\documentclass[unnumsec,webpdf,contemporary,large]{oup-authoring-template}%

\usepackage{booktabs}
\usepackage[ruled,vlined,linesnumbered,algo2e]{algorithm2e}
\usepackage{comment}
\usepackage{tikz}
\usepackage{caption}
\usepackage{tabularx}

\usepackage{amsthm}
\newtheorem{assumption}{Assumption}
\newtheorem{lemma}{Lemma}
\usetikzlibrary{arrows.meta, positioning}

\tikzset{
  node style/.style={circle, draw, minimum size=1cm, inner sep=0pt},
  arrow style/.style={-{Stealth}, thick},
  dashed arrow/.style={arrow style, dashed},
  bidir/.style={<->, thick, dashed, bend left=40}
}
\graphicspath{{Fig/}}

\usepackage[mathlines, switch]{lineno}

\theoremstyle{thmstyleone}%
\newtheorem{theorem}{Theorem}
\theoremstyle{thmstyletwo}%
\newtheorem{remark}{Remark}%
\theoremstyle{thmstylethree}%
\newtheorem{definition}{Definition}
\begin{document}

\journaltitle{Bioinformatics}
\DOI{DOI HERE}
\copyrightyear{2026}
\pubyear{2026}
\access{Advance Access Publication Date: Day Month Year}
\appnotes{Paper}

\firstpage{1}


\title[Causal ASCEND]{ Causal ASCEND: Scalable Two-tier Causal Discovery on High Dimensional Multi-omics Data}

\author[1,$\ast$]{Stephen Asiedu \ORCID{0009-0005-8885-5381}}
\author[1]{David Watson}

\authormark{Asiedu et al.}

\address[1]{\orgdiv{Department of Informatics}, \orgname{King's College London}, \orgaddress{\street{Bush House, 30 Aldwych}, \postcode{WC2B 4BG}, \state{London}, \country{United Kingdom}}}
\corresp[$\ast$]{Corresponding authors. \href{email:stephen.asiedu@kcl.ac.uk}{stephen.asiedu@kcl.ac.uk}, \href{email:stephen.asiedu@kcl.ac.uk}{david.watson@kcl.ac.uk}}

\newcommand{\indep}{\perp \!\!\! \perp}
\newcommand{\dep}{\not\! \perp\!\!\!\perp}



\abstract{\textbf{Motivation:} \\
Biological systems exhibit a hierarchical structure, characterised by directed flow from upstream regulators to downstream effects. Although this ordering provides a natural scaffold for causal inference, most causal discovery and GRN methods either ignore the tiered organisation or condition on all upstream variables, which becomes infeasible for high-dimensional omics data.\\
\vspace{5pt}
\textbf{Results:}\\
We present ASCEND (Ancestral Scalable Causal discovEry via iNherited Descent), a constraint-based framework that leverages known two-tiered structure to enable genome-scale causal discovery. ASCEND introduces a divide-and-conquer strategy that maintains dynamically updated ancestral conditioning sets for each downstream variable, dramatically reducing the number of conditional independence tests required, and achieves polynomial-time complexity where traditional approaches face exponential blow-up. Through extensive simulations and real biological data, we demonstrate that ASCEND accurately recovers ancestral relationships, scales properly and much faster, and outperforms existing gene regulatory network inference methods in both causal precision and computational efficiency. The algorithm's ability to resolve directionality makes it particularly suited for integrating multi-omic data where upstream regulators (e.g., SNPs, methylation sites) and downstream responses (e.g., gene expression) are measured jointly.}


\maketitle

\section{Introduction}
The central goal of systems biology is to transition from a descriptive catalogue of molecular components to a functional map of the causal mechanisms governing life \citep{rebai2017causality, Lynch2021The, chevalley2022causalbench, glymour2019review, Ayyanathan2014AssessingCR, Hu2018ApplicationOC}. With the maturation of high-throughput sequencing, we now possess unprecedented multi-omic profiles ranging from genomic variation to transcriptomic and proteomic responses \citep{Manel2016Genomic, He2017Big, Neale2019Gene, Abu-Elmagd2022Editorial:}. However, the sheer dimensionality of these datasets has created a paradox of data richness with little knowledge of the true governing structure \citep{Bates2020Causal}. While we can observe thousands of simultaneous molecular shifts, distinguishing the primary drivers of disease from their downstream effects remains a formidable challenge.

Biological systems are fundamentally hierarchical. This organisation is not a mere byproduct of complexity but is rooted in the central dogma and the directional flow of information from inherited background variables to foreground variables \citep{Danchin2007The, Veenstra2012Metabolomics:, Neale2019Gene}.  For example, genetic variations biologically precede transcriptomic states. In this ``two-tiered" landscape, background variables act as an upstream scaffold that causally precedes foreground variables \citep{watson2022causal}. This inherent causal ordering provides a natural constraint that should drastically simplify the search for Gene Regulatory Networks (GRNs).

Exploiting this structure, however, requires paired multi-omic measurements on the same samples: a stronger data requirement than the ``transcriptomics-only" setting in which most GRN methods operate. Where such
matched background and foreground layers are available, the tiered ordering supplies causal information that association-based methods cannot access; yet
most modern frameworks still do not exploit it. Traditional GRN inference methods frequently treat molecular variables as homogeneous \citep{pratapa2020benchmarking, badia2023gene}. Popular GRN approaches such as information-theoretic methods (e.g., ARACNE \citep{margolin2006aracne}),
tree-based models (e.g., GENIE3 \citep{huynh2010inferring}), recover statistical associations but cannot resolve directionality. eQTL-anchored methods such as TRIGGER \citep{chen2007harnessing} do orient causality from genetic instruments, but require a detectable instrument for every gene and derive their guarantees from genotype randomisation,
conditions rarely met in observational human multi-omic data. Constraint-based causal discovery algorithms \citep{spirtes2001causation, kalisch2007estimating} offer theoretical guarantees but rely on conditional independence (CI) tests whose complexity scales poorly in high-dimensional settings \citep{shah2023causal}. This problem is exacerbated in genomics, where thousands of upstream variables (e.g., SNPs, chromatin features) and thousands of downstream traits coexist, making conditioning on the full background layer computationally prohibitive. While formal causal discovery algorithms \citep{glymour2019review} offer the theoretical rigour to identify directionality, they encounter a ``computational wall" in the omics setting. As the number of variables grows, the complexity of conditional independence (CI) testing grows exponentially \citep{sandor2025efficient, asiedu2025multi, dai2024gene}. Naively conditioning on the entire high-dimensional upstream layer results into scalability issues \citep{shah2023causal}. While some recent methods begin to incorporate biological hierarchy \citep{mao2025halo, watson2022causal, shah2023causal}, a scalable, general framework that leverages this knowledge to make genome-wide causal discovery tractable is still lacking.

\begin{figure*}[t]
\centering
\includegraphics[width=0.75\textwidth]{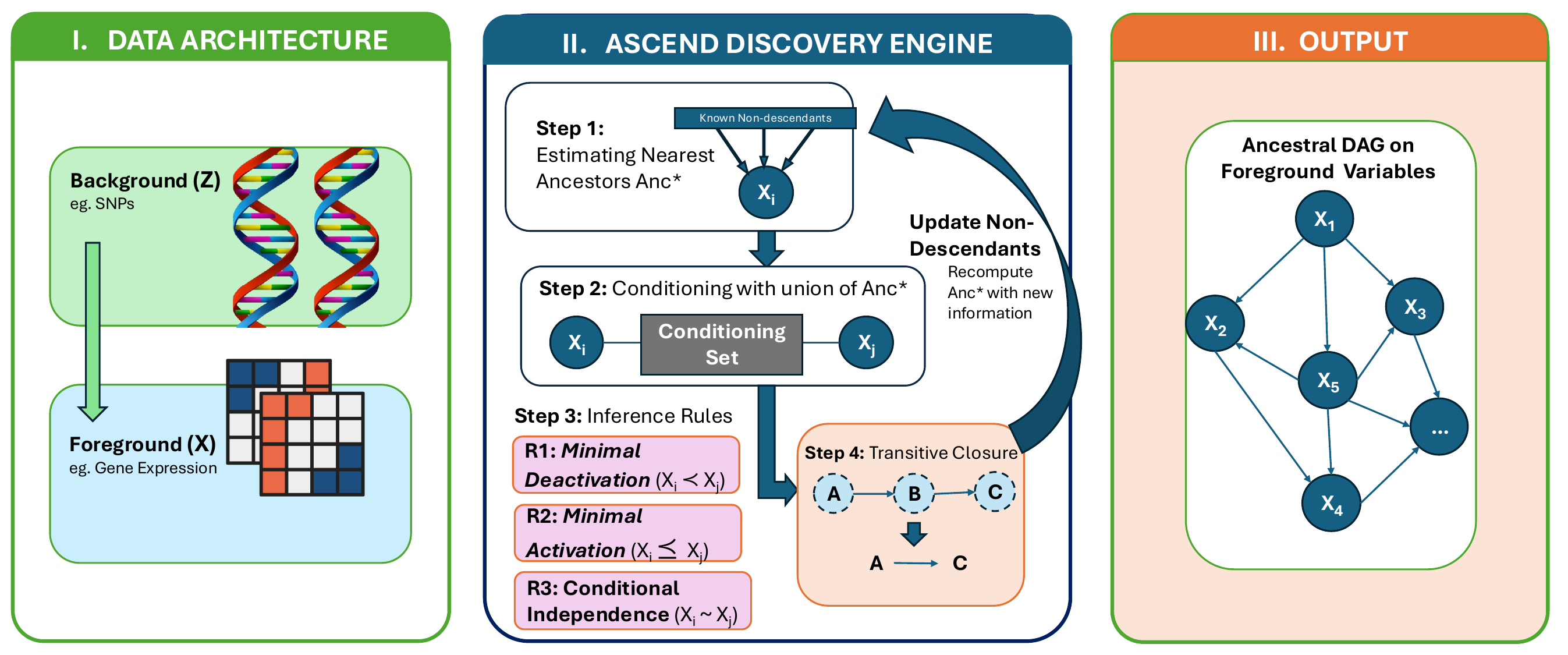}
\caption{\textbf{ASCEND workflow.} (I) The two-tiered data structure where some background $Z$ causally precede foreground $X$. (II) Structure learning: Test pairwise conditional independencies among $\mathbf{X}$ variables while conditioning only on ${Anc}^*(X, \mathcal{T}_{X_i}) \cup {Anc}^*(X, \mathcal{T}_{X_j})$. (III) Final Output after iterations and closure depicting ancestral relationship. Edges may represent direct or indirect causes.}
\label{fig:workflow}
\end{figure*}

To bridge this gap, we introduce ASCEND (Ancestral Scalable Causal discovEry via iNherited Descent), a constraint-based causal discovery framework designed for high-dimensional, two-tiered biological systems. ASCEND assumes a known causal ordering between background and foreground variables, reflecting common biological hierarchies. Our method adopts a divide-and-conquer strategy that decomposes global structure learning into tractable local problems, allowing it to scale efficiently to omics-scale datasets.

Where the earlier TRIGGER \citep{chen2007harnessing} orients causality from genetic instruments and the more recent confounder blanket learner (CBL) \citep{watson2022causal} conditions on the full background set, ASCEND instead maintains and updates a set of nearest ancestors (inherited descent), which we show as a minimally valid conditioning set for conditional independence testing (see Supplementary material for proof), for each foreground variable. This dramatically reduces the number of conditional independence tests, improving efficiency and preserving soundness under standard causal assumptions. By exploiting the known hierarchical structure, ASCEND achieves substantial computational gains over existing two-tiered and GRN inference approaches. Our workflow is summarised in Fig. \ref{fig:workflow}.

We evaluate ASCEND on simulated datasets and real biological data, demonstrating improved computational efficiency and recovery of hierarchical causal relationships. Comparisons against standard GRN inference tools and two-tiered causal methods highlight the advantages of using these biologically plausible conditioning sets for scalable causal discovery. Our results illustrate how leveraging natural biological hierarchies can make genome-scale causal inference feasible, interpretable, and biologically informative.

The remainder of the paper reports benchmarks against GRN inference tools (Section~\ref{sec:bench-synthetic}), applications to real \textit{Drosophila}
and multi-omic data (Section~\ref{sec:drosophila}), and comparisons
against CBL and other causal-discovery baselines
(Sections~\ref{sec:bench-cbl}--\ref{sec:bench-causal}).

\section{Results}
\label{sec:results}

\subsection{Benchmarking on synthetic data}
\label{sec:bench-synthetic}

We begin with simulated data to allow for explicit control of causal structure.
We simulated two-tiered biological systems using the linear-Gaussian data-generating process detailed in Supplementary Section~S\ref{sec:simulation}. Each replicate generated 20 background variables ($\mathbf{Z}$) and 15 foreground variables ($\mathbf{X}$); upstream-to-downstream cross-tier edges were drawn with probability $0.20$ and within-tier edges from an Erd\H{o}s--R\'enyi DAG controlled by a sparsity parameter $sp$. We swept twelve conditions formed by the Cartesian product of $n \in \{1000, 2000\}$ samples,
$sp \in \{0.5,\, 0.7,\, 0.9\}$, and signal strength $R^2 \in \{0.5,\, 0.7\}$, running $50$ independent replicates per cell. All methods received the same input matrix $(\mathbf{Z} \cup \mathbf{X})$ and
were evaluated against the same ancestral skeleton ground truth derived from the simulator's DAG.

We compared ASCEND against three widely used GRN inference methods spanning distinct methodological families: random forest regression (GENIE3 \citep{huynh2010inferring}), information-theoretic (ARACNe \citep{margolin2006aracne}), and weighted correlation networks (WGCNA \citep{langfelder2008wgcna}). For each method we report the area under the precision-recall curve (AUPR), which is the appropriate ranking metric under the severe class imbalance typical of biological networks and is the standard adopted by the BEELINE benchmarking suite \citep{pratapa2020benchmarking}; AUROC; and the F1 score at \emph{matched edge count}, in which competitors are thresholded to claim the same number of edges that ASCEND claims on the same replicate. The matched-K design controls for differences in operating-point selection and removes any threshold-fairness confound. Significance was assessed using one-sided paired Wilcoxon signed-rank tests (paired by replicate), with Benjamini–Hochberg correction across all ASCEND-versus-competitor comparisons within each metric family.

Table~\ref{tab:bench-primary} summarises the primary cell ($n=2000$, $sp=0.9$, $R^2=0.7$) and Figure~\ref{fig:bench-headline} visualises the matched-K F1 across all twelve cells. At the primary cell, ASCEND achieves F1 $=0.589 \pm 0.023$ versus the next-best $0.358 \pm 0.020$ (GENIE3), a $64\%$ relative improvement, with precision $0.808$ versus $0.471$ at the same edge count. AUPR is correspondingly higher
($0.530 \pm 0.023$ versus $0.428 \pm 0.019$, and q-value $q=0.002$). The F1 advantage is statistically significant in every one of the twelve
conditions tested (all $q<0.01$; Table~\ref{tab:bench-wilcoxon}), with median per-replicate F1 differences ranging from $+0.000$ in the dense, weak-signal regime to $+0.481$ in the sparse, high-signal regime. Across the sparsity sweep, ASCEND's F1 advantage widens as the true graph becomes sparser: when the underlying ancestral structure is sparse, constraint-based testing extracts more signal than continuous edge-weight rankings. This is the regime most relevant to real GRNS, which are known to be sparse \citep{barabasi2004network}.

Among the methods evaluated, only ASCEND outputs orientations. We report the fraction of claimed ancestral edges whose direction agrees with the ground-truth DAG (Table~\ref{tab:bench-primary}, last but one column). At the sparsest cells ($sp=0.9$), ASCEND's direction accuracy reaches $77\%$--$82\%$ ($77.4\% \pm 2.7\%$ at the primary cell), well above the $50\%$ chance baseline in correlational methods. In denser regimes ($sp \in \{0.5,\,0.7\}$) direction accuracy degrades to $58\%$--$65\%$, still above chance but by smaller margins. This pattern reflects a known property of constraint-based causal discovery:
orientation rules rely on conditional-independence patterns that become harder to detect when the underlying graph has many parallel paths. Despite performing repeated conditional-independence testing, ASCEND is substantially faster than GENIE3. Mean wall-clock runtime per replicate at $n=2000$ ranges from $0.30$~s ($sp=0.9$, $R^2=0.5$)
to $0.62$~s ($sp=0.5$, $R^2=0.7$), versus $144$--$527$~s for GENIE3 in the
same cells---a speedup of more than two orders of magnitude (${\sim}230$--$870\times$). WGCNA runs in comparable sub-second time but recovers substantially fewer edges, and ARACNe's runtime scales unfavourably at $n=2000$ (${\gtrsim}110$~s). The runtime gap widens further with sample size; the dynamic ancestor sets used by ASCEND keep its conditioning sets small, so per-replicate cost scales with the sparsity of the discovered structure rather than the full dimensionality of the background
layer.

We highlight two regimes where ASCEND does not dominate. First, in the moderate-signal regime ($R^2=0.5$), ASCEND's AUPR advantage over GENIE3 narrows and is not significant in every cell (e.g.\ $q>0.05$ at $n=1000$, $sp=0.5$ and at $n=2000$, $sp=0.9$), although ASCEND remains significantly higher on F1 throughout. The two methods rank candidate edges comparably well in this regime, but ASCEND's discrete output (a sparse set
of high-precision claims) gives it a substantial F1 advantage at matched thresholds. Second, ASCEND has lower AUROC than GENIE3 at the sparsest,
low-signal cells ($0.72$ versus $0.77$ at $sp=0.9$, $R^2=0.5$), though the two are essentially tied at the primary cell ($0.731$ versus $0.732$). AUROC weights performance against the full ranking rather than the imbalanced positive class, and is known to be optimistic in this setting \citep{saito2015precision}. We retain AUROC in our reporting for completeness but adopt AUPR and matched-K F1 as the primary metrics, consistent with BEELINE.

\begin{table*}[t]
\centering
\caption{\textbf{Primary benchmark cell: $n=2000$, $sp=0.9$, $R^2=0.7$.}
Mean $\pm$ standard error across 50 replicates. Best performance per
metric in bold. Runtime is mean wall-clock
seconds per replicate}
\label{tab:bench-primary}
\small
\begin{tabular}{@{}lcccccc@{}}
\toprule
\textbf{Method} & \textbf{AUPR} & \textbf{AUROC} & \textbf{F1}
  & \textbf{Precision} & \textbf{Dir.\ acc.} & \textbf{Runtime (s)} \\
\midrule
ASCEND  & $\mathbf{0.530 \pm 0.023}$ & $\mathbf{0.731 \pm 0.014}$
        & $\mathbf{0.589 \pm 0.023}$ & $\mathbf{0.808}$
        & $\mathbf{0.774 \pm 0.027}$  & $0.4$ \\
GENIE3  & $0.428 \pm 0.019$          & $\mathbf{0.732 \pm 0.011}$
        & $0.358 \pm 0.020$          & $0.471$
        & N/A                         & $149.2$ \\
ARACNe  & $0.154 \pm 0.009$          & $0.500 \pm 0.000$
        & $0.173 \pm 0.016$          & $0.146$
        & N/A                         & $110.4$ \\
WGCNA   & $0.224 \pm 0.016$          & $0.577 \pm 0.009$
        & $0.219 \pm 0.017$          & $0.250$
        & N/A                         & $0.8$  \\
\bottomrule
\end{tabular}
\end{table*}

\begin{table}[t]
\centering
\caption{\textbf{Full sweep: F1 across all twelve conditions.} Mean F1 $\pm$ standard error across 50 replicates. Bold indicates top performance. ASCEND statistically outperforms all competitors ($q<0.01$).}
\label{tab:bench-f1-full}
\small
\setlength{\tabcolsep}{4pt}
\begin{tabular}{@{}lcccc@{}}
\toprule
\textbf{$sp$ / $n$} & \textbf{ASCEND} & \textbf{GENIE3} & \textbf{ARACNe} & \textbf{WGCNA} \\
\midrule
\multicolumn{5}{l}{\textit{Signal Strength ($R^2 = 0.5$)}} \\
\midrule
$0.5$ / $1000$ & $\mathbf{0.20 \pm 0.02}$ & $0.18 \pm 0.02$ & $0.15 \pm 0.01$ & $0.15 \pm 0.01$ \\
$0.5$ / $2000$ & $\mathbf{0.31 \pm 0.02}$ & $0.28 \pm 0.02$ & $0.25 \pm 0.02$ & $0.26 \pm 0.02$ \\
$0.7$ / $1000$ & $\mathbf{0.29 \pm 0.02}$ & $0.25 \pm 0.02$ & $0.20 \pm 0.02$ & $0.20 \pm 0.02$ \\
$0.7$ / $2000$ & $\mathbf{0.33 \pm 0.02}$ & $0.28 \pm 0.02$ & $0.22 \pm 0.02$ & $0.23 \pm 0.01$ \\
$0.9$ / $1000$ & $\mathbf{0.54 \pm 0.03}$ & $0.39 \pm 0.02$ & $0.13 \pm 0.01$ & $0.21 \pm 0.02$ \\
$0.9$ / $2000$ & $\mathbf{0.56 \pm 0.02}$ & $0.41 \pm 0.02$ & $0.15 \pm 0.02$ & $0.23 \pm 0.02$ \\
\midrule
\multicolumn{5}{l}{\textit{Signal Strength ($R^2 = 0.7$)}} \\
\midrule
$0.5$ / $1000$ & $\mathbf{0.29 \pm 0.02}$ & $0.26 \pm 0.02$ & $0.23 \pm 0.02$ & $0.24 \pm 0.02$ \\
$0.5$ / $2000$ & $\mathbf{0.33 \pm 0.02}$ & $0.30 \pm 0.03$ & $0.28 \pm 0.02$ & $0.28 \pm 0.02$ \\
$0.7$ / $1000$ & $\mathbf{0.33 \pm 0.02}$ & $0.28 \pm 0.02$ & $0.24 \pm 0.02$ & $0.24 \pm 0.02$ \\
$0.7$ / $2000$ & $\mathbf{0.40 \pm 0.02}$ & $0.33 \pm 0.02$ & $0.29 \pm 0.02$ & $0.30 \pm 0.02$ \\
$0.9$ / $1000$ & $\mathbf{0.57 \pm 0.03}$ & $0.30 \pm 0.02$ & $0.19 \pm 0.01$ & $0.21 \pm 0.02$ \\
$0.9$ / $2000$ & $\mathbf{0.59 \pm 0.02}$ & $0.36 \pm 0.02$ & $0.17 \pm 0.02$ & $0.22 \pm 0.02$ \\
\bottomrule
\end{tabular}
\end{table}

\begin{table}[t]
\centering
\caption{\textbf{Paired Wilcoxon signed-rank tests, ASCEND versus each
competitor on F1 at matched edge count.}
One-sided alternative (ASCEND $>$ competitor), paired by replicate,
Benjamini--Hochberg corrected across all 36 comparisons per metric.
Median difference is per-replicate median(ASCEND~F1 $-$ competitor~F1).
ASCEND's F1 is significantly higher than every competitor in every
condition tested. A median difference of $0.000$ with significant test
($n=1000$, $sp=0.5$, $R^2=0.5$ vs GENIE3) indicates that ASCEND wins on
the majority of paired replicates by small margins; the median pair has
zero difference but the rank-sum is significantly shifted toward ASCEND.
$^{***}q<0.001$; $^{**}q<0.01$.}
\label{tab:bench-wilcoxon}
\small
\begin{tabular}{@{}llrrrl@{}}
\toprule
$R^2$ & $sp$ / $n$ & \textbf{vs.\ GENIE3} & \textbf{vs.\ ARACNe}
                  & \textbf{vs.\ WGCNA} & \\
\midrule
\multirow{6}{*}{$0.5$}
& $0.5$ / $1000$ & $+0.000^{***}$   & $+0.058^{***}$  & $+0.033^{***}$ &  \\
& $0.5$ / $2000$ & $+0.022^{***}$   & $+0.066^{***}$  & $+0.068^{***}$ &  \\
& $0.7$ / $1000$ & $+0.036^{***}$   & $+0.101^{***}$  & $+0.090^{***}$ &  \\
& $0.7$ / $2000$ & $+0.043^{***}$   & $+0.120^{***}$  & $+0.100^{***}$ &  \\
& $0.9$ / $1000$ & $+0.160^{***}$   & $+0.427^{***}$  & $+0.314^{***}$ &  \\
& $0.9$ / $2000$ & $+0.174^{***}$   & $+0.471^{***}$  & $+0.308^{***}$ &  \\
\midrule
\multirow{6}{*}{$0.7$}
& $0.5$ / $1000$ & $+0.028^{***}$   & $+0.046^{***}$  & $+0.029^{***}$ &  \\
& $0.5$ / $2000$ & $+0.027^{***}$   & $+0.047^{***}$  & $+0.041^{***}$ &  \\
& $0.7$ / $1000$ & $+0.056^{***}$   & $+0.078^{***}$  & $+0.090^{***}$ &  \\
& $0.7$ / $2000$ & $+0.056^{***}$   & $+0.092^{***}$  & $+0.096^{***}$ &  \\
& $0.9$ / $1000$ & $+0.273^{***}$   & $+0.481^{***}$  & $+0.407^{***}$ &  \\
& $0.9$ / $2000$ & $+0.229^{***}$   & $+0.447^{***}$  & $+0.339^{***}$ &  \\
\bottomrule
\end{tabular}
\end{table}

\begin{figure}
\centering
\includegraphics[width=0.5\textwidth]{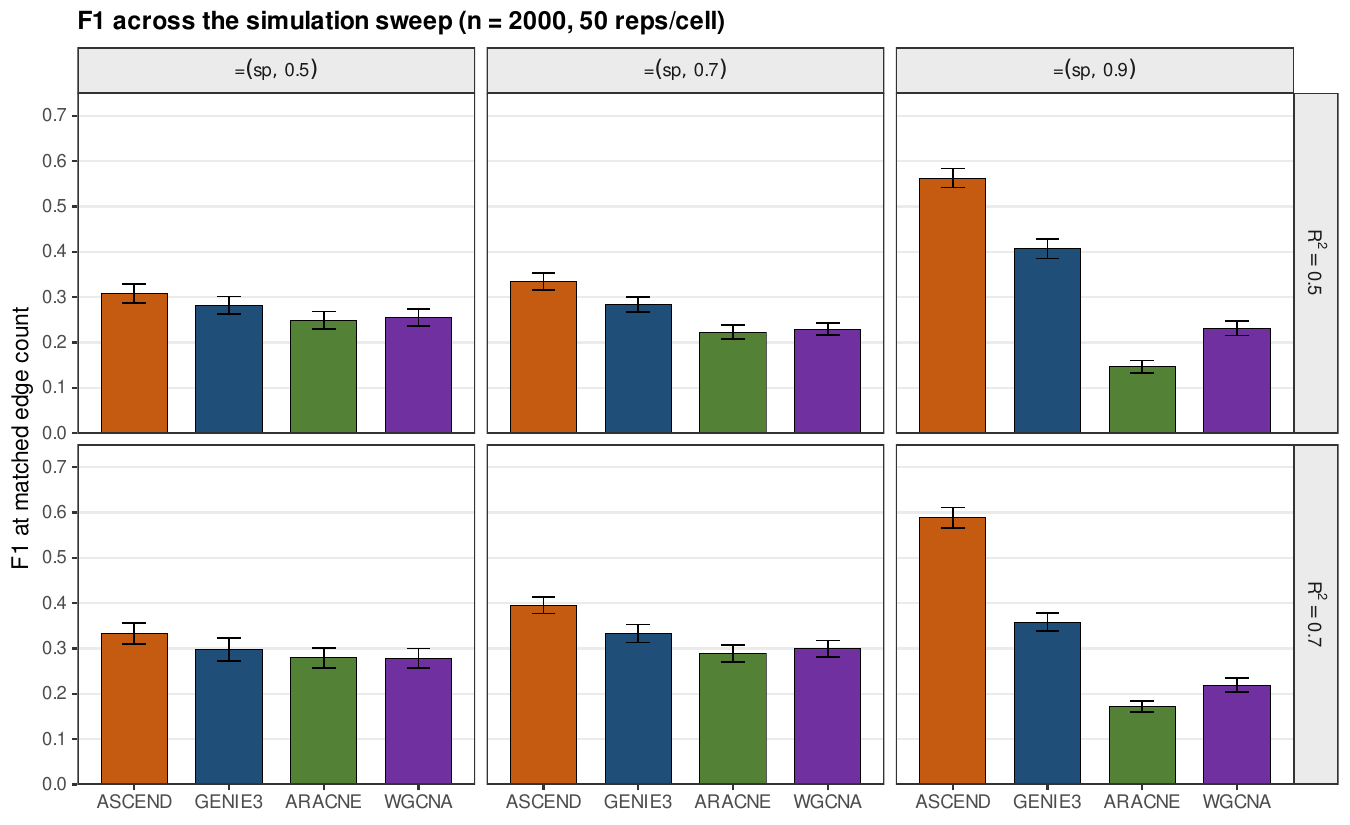}
\caption{\textbf{F1 at matched edge count across the full sweep.}
Mean F1@K over 50 replicates per cell, with rows grouped by signal
strength ($R^2 \in \{0.5,\, 0.7\}$) and columns by graph sparsity
($sp \in \{0.5,\, 0.7,\, 0.9\}$) at $n=2000$. ASCEND attains the highest
F1 in every cell, with the advantage widening as the underlying graph
becomes sparser. Error bars are standard errors across replicates; all
ASCEND-versus-best-competitor differences are significant at
$q<0.01$ by paired one-sided Wilcoxon test
(Table~\ref{tab:bench-wilcoxon}).}
\label{fig:bench-headline}
\end{figure}

\subsection{The Drosophila Genetic Reference Panel (DGRP) Benchmark } \label{sec:drosophila}

To evaluate ASCEND’s performance in ancestral discovery, we benchmarked the framework using the Drosophila Genetic Reference Panel (DGRP), a population-based resource of 200 inbred lines of \textit{Drosophila melanogaster}. These lines originated from a single natural population from Raleigh, North Carolina. Through 20 generations of full-sib mating, these lines have achieved near-total homozygosity, providing a stabilised genetic architecture where molecular phenotypic variation is driven by natural polymorphisms rather than stochastic laboratory mutations. We integrated genomic variants as background $Z$ variables from the DGRP Freeze 2.0 \citep{mackay2012drosophila} with transcriptomic profiles as $X$ from adult males \citep{everett2020gene}, totalling 12806 annotated genes, which represent about $75\%$ of the known Drosophila genome.

\subsubsection{Experimental Design and Standardisation}
We implemented a high-performance preprocessing pipeline to harmonise raw omics matrices for ASCEND’s conditional independence testing framework. The DGRP genotype matrix was subjected to a complete-case filter, retaining only SNPs with 100\% call rates across all 200 lines. We restricted the feature space to common variants (Minor Allele Frequency, $0.2 < MAF < 0.5$), to maintain statistical power and ensure stable covariance estimates.  This eliminates the "sparse data" problem often encountered in high-dimensional discovery, where low-frequency alleles lead to unstable covariance estimates and inflated Type I errors. For the transcriptomic data, we applied a two-step filter to the 18,140 transcribed regions: i) Expression Thresholding: We retained genes with a mean $\log_2(FPKM)$ between 4 and 10 to ensure biological relevance and minimise heteroscedasticity associated with low-abundance transcripts. ii) Variance Selection: We selected the top 250 genes by variance. This selection prioritises the most dynamic regulatory components, those most likely to be under active genetic control (eQTLs). Finally, we aligned the 200 male samples common to both the genomic and transcriptomic datasets to ensure cross-modal consistency.

\subsubsection{Identification of Genomic Causal Hubs}
To transition from statistical associations to a structural understanding of the DGRP transcriptome, we employed ASCEND, which oriented 880 directed causal edges and 61 undirected edges among the selected 250 transcripts, yielding a network density of 3.02\%.  Mapping the causal out-degree of each gene back to its genomic coordinates revealed a non-random regulatory landscape concentrated on chromosome 2R.

The highest-outdegree node was \textit{GNBP-like3} (FBgn0034511, CG13422), located at 2R:20.5\,Mb, with 51 directed downstream edges. \textit{GNBP-like3} encodes a $\beta$-glucan-binding pattern recognition receptor that, together with GNBP3, feeds into Toll pathway activation via a positive-feedback induction loop \citep{lu2024suppression}. Several of the next-highest out-degree nodes are established Toll-pathway effectors: Metchnikowin (\textit{Mtk}, FBgn0014865, out-degree 48), a proline-rich antimicrobial peptide with documented allelic variation under balancing selection \citep{perlmutter2024single}, and \textit{BomBc1} (FBgn0034328, out-degree 27), a member of the ten-gene Bomanin cluster at cytological position 55C whose collective deletion phenocopies a Toll pathway null \citep{clemmons2015effector, lindsay2018short}. \textit{CG10911} (FBgn0034295, out-degree 49) remains functionally uncharacterized and is reported here as a candidate for follow-up. The recovery of a coherent, literature-recognized immune module (a pattern recognition recepter alongside its downstream antimicrobial peptide and Bomanin effectors) from purely population-genetic data is a meaningful positive signal. It indicates that ASCEND is tracking a real axis of genetic variation in this population. Table \ref{tab:regulator_genes} summarises the top 7 regulators discovered by ASCEND and their associated biological function. Fig \ref{fig:manhattan} summarises our discovery. We present full detailing of all the hubs we discovered in the supplementary material. These hubs can be used as prior knowledge for further biological experimentations.

Identifying these hubs moves the analysis beyond a static ``parts list'' of the genome toward a functional, directed model, with several practical consequences. Together, these findings establish a locus as a high-leverage target for experimental follow-up like CRISPR-based perturbation.

\begin{figure}
    \centering
    \includegraphics[width=0.99\linewidth]{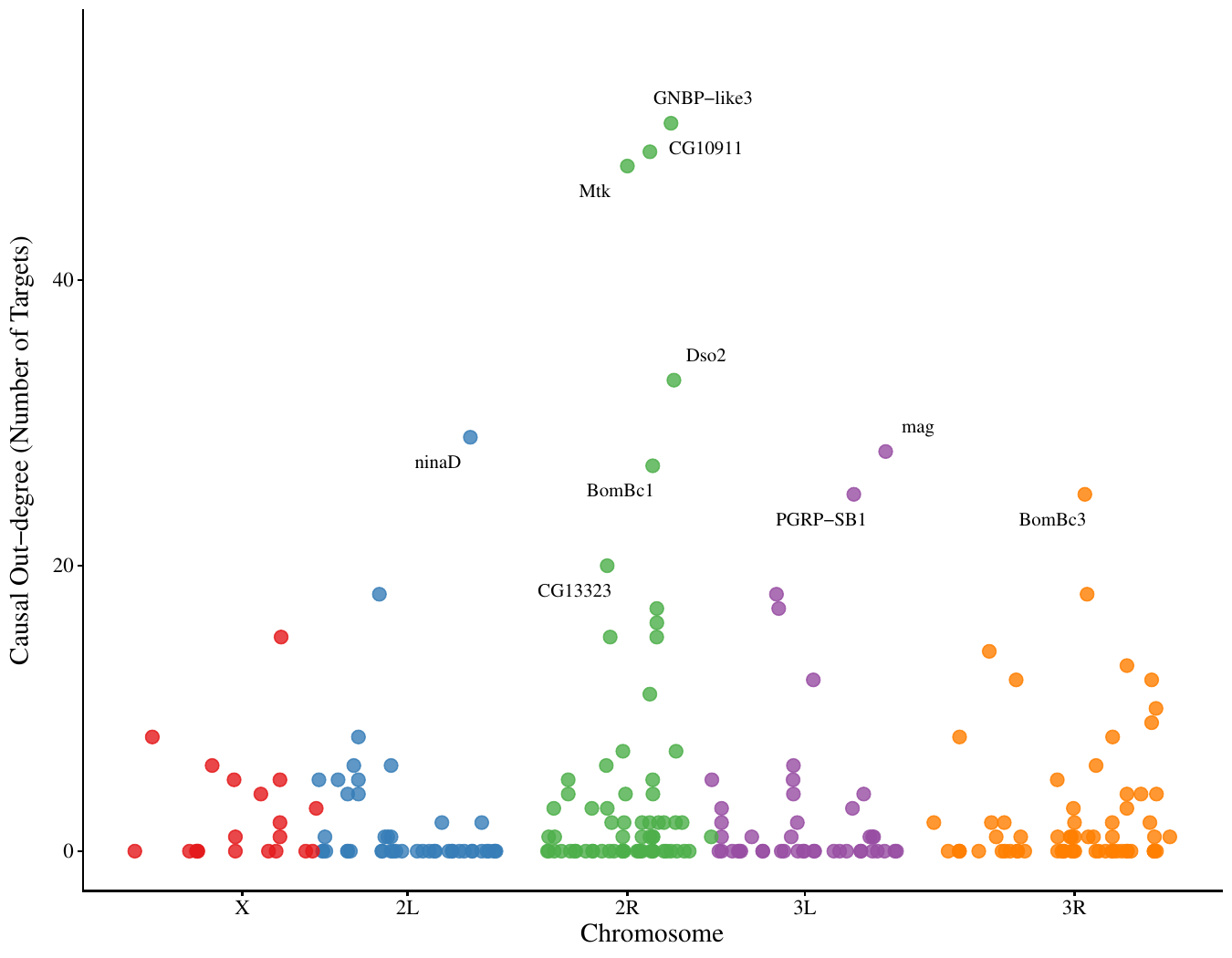}
    \caption{The genomic distribution of causal influence scores. The y-axis represents the causal out-degree (the number of downstream genes controlled by a specific locus).}
    \label{fig:manhattan}
\end{figure}

\begin{table*}[h]
    \centering
    \caption{Top Upstream Master Regulators Discovered by ASCEND and their Chromosomal Distribution}
    \label{tab:regulator_genes}
    \small
    \setlength{\tabcolsep}{5pt}
    \begin{tabularx}{\linewidth}{@{} l c c c >{\raggedright\arraybackslash}X @{}}
        \toprule
        \textbf{ Hub} & \textbf{FlyBase ID} & \textbf{Chr} & \textbf{Out-Degree} & \textbf{Primary Biological Association} \\
        \midrule
        GNBP-like3 & FBgn0034511 & 2R & 51 & Gram-negative pathogen recognition; Toll activation \\
        \addlinespace
        CG10911    & FBgn0034295 & 2R & 49 & Core transcriptomic driver; uncharacterized eQTL \\
        \addlinespace
        Mtk        & FBgn0014865 & 2R & 48 & Metchnikowin; antifungal/antibacterial peptide defense \\
        \addlinespace
        Dso2       & FBgn0067905 & 2R & 33 & Disconnected-subunit-like 2; immune effector response \\
        \addlinespace
        ninaD      & FBgn0002939 & 2L & 29 & Class B scavenger receptor; midgut carotenoid transport \\
        \addlinespace
        mag        & FBgn0036996 & 3L & 28 & \textit{magistrale}; protein kinase acting in tracheal development \\
        \addlinespace
        BomBc1     & FBgn0034328 & 2R & 27 & Bomanin B-core 1; antimicrobial humoral response \\
        \bottomrule
    \end{tabularx}
\end{table*}

\subsection{Comparison with existing directional gene regulatory network inference methods}

To evaluate the performance of ASCEND against established approaches for causal gene regulatory network inference, we compared it with two representative methods: TRIGGER~\citep{chen2007harnessing}, one of the earliest methods to exploit naturally occurring genetic variation for causal network reconstruction, and BFCS (Bayes Factors of Covariance Structures)~\citep{bucur2018bayesian}, a Bayesian causal inference framework that infers directed regulatory interactions using instrumental-variable models. All three methods were applied to the same yeast eQTL dataset distributed with both the TRIGGER and BFCS package, comprising expression measurements for 6,216 genes across 112 segregants together with 3,244 genetic markers.

Unlike ASCEND, which infers causal relationships directly from conditional independence structure, both TRIGGER and BFCS require genetic instruments linked to putative regulators. Consequently, the methods estimate related but not identical causal quantities: ASCEND reconstructs the ancestral causal ordering of the network, whereas TRIGGER and BFCS assign confidence scores to pairwise regulator--target relationships. To ensure a fair comparison, we therefore evaluated both direct regulatory interactions and ancestral (reachability) relationships separately.

\subsubsection{Construction of the benchmark network}

Evaluation was performed against experimentally curated transcriptional interactions obtained from the YEASTRACT+ database \cite{teixeira2023yeastract}. Gene identifiers from the expression dataset, TRIGGER matrices and curated regulatory interactions were first mapped to systematic ORF identifiers to ensure a common gene space. Regulators represented in both the expression data and curated network with at least three measurable downstream targets were retained.

To construct a biologically meaningful benchmark while maintaining sufficient statistical power, we selected the highest-degree curated transcription factors together with their experimentally validated targets, yielding a network of 150 genes containing 14 regulators and 298 curated direct regulatory interactions. Because ASCEND predicts transitive ancestral relationships rather than only direct edges, we additionally computed the transitive closure of the curated network restricted to this induced subgraph, producing a second benchmark consisting of 368 ancestral regulatory relationships. 

\subsubsection{Recovery of curated regulatory relationships}

We used Precision, recall and F1 score against both the direct and ancestral benchmarks, together with per-regulator Jaccard similarity between predicted and curated target sets as shown in Fig \ref{fig:ascend_trigger_bfcs}. ASCEND consistently achieved the highest performance across all evaluation metrics. On the direct-edge benchmark, ASCEND attained a precision of 34.8\%, substantially exceeding both TRIGGER (17.4\%) and BFCS (13.0\%), while also achieving the highest recall and F$_1$ score.

\begin{figure}[t]
\centering
\includegraphics[width=0.5\textwidth]{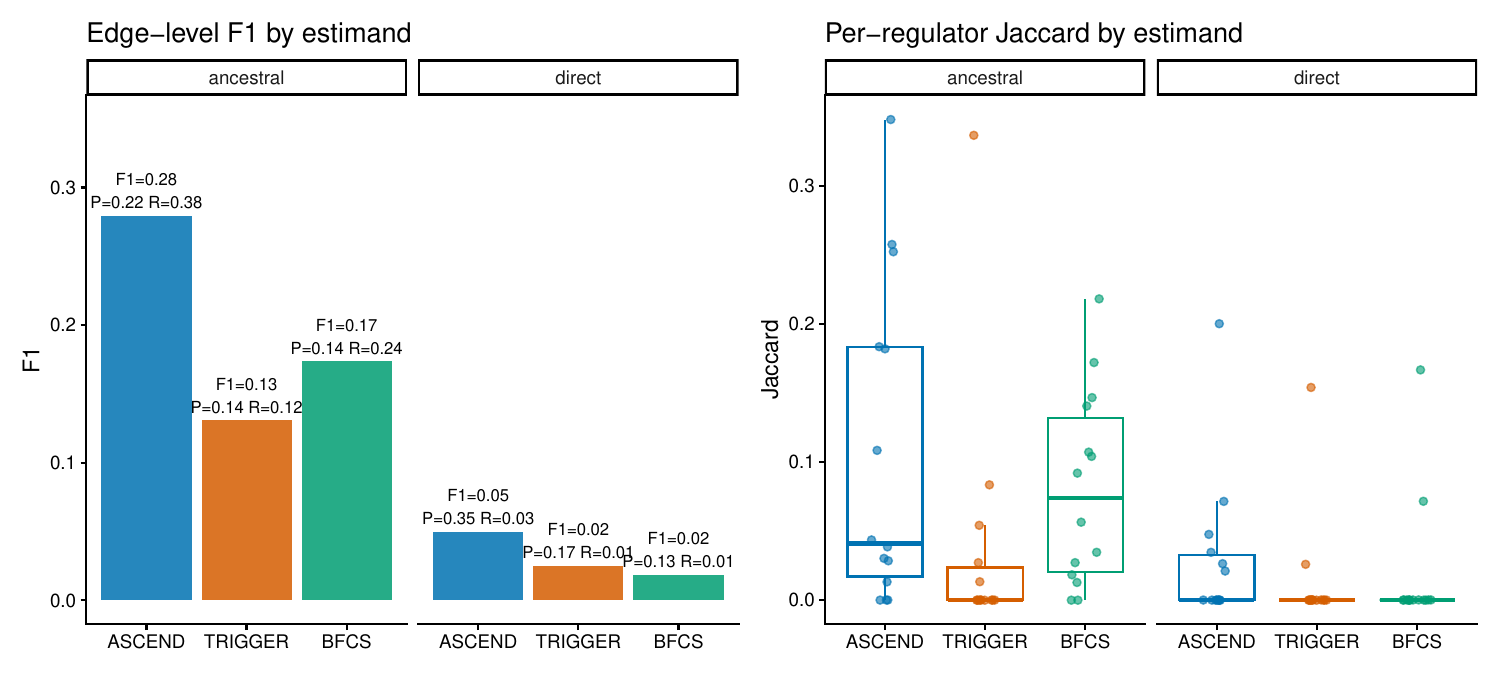}
\caption{Comparison of ASCEND, TRIGGER and BFCS. Left: edge-level precision, recall and F$_1$ scores for direct and ancestral regulatory relationships. Right: per-regulator Jaccard similarity between predicted and curated regulons.}
\label{fig:ascend_trigger_bfcs}
\end{figure}

The advantage was more pronounced on the ancestral benchmark, which more naturally reflects the causal relationships estimated by ASCEND. ASCEND achieved a precision of 22.1\%, a recall of 38.0\%, and an F$_1$ score of 0.279, compared with 0.174 for BFCS and 0.131 for TRIGGER. Thus, ASCEND improved F$_1$ performance by approximately 60\% over BFCS and more than twofold over TRIGGER. Similar trends were observed when comparing regulator-specific target sets: ASCEND produced the highest mean Jaccard similarity on both the direct and ancestral benchmarks.

The three methods showed moderate agreement in the regulatory relationships they recovered. ASCEND shared 192 predicted ancestral interactions with BFCS, compared with 117 shared between BFCS and TRIGGER and 53 shared between ASCEND and TRIGGER, indicating that although ASCEND and BFCS are based on fundamentally different inference principles, they recover a substantial overlapping subset of biologically plausible regulatory relationships.

\section{Comparison against the confounder blanket learner}
\label{sec:bench-cbl}

The closest methodological relative of ASCEND is the confounder blanket learner (CBL) of \citet{watson2022causal}, which also exploits a known two-tiered ordering between background and foreground variables but conditions on the \emph{entire} background layer when testing pairwise ancestral relations. Because both methods address the same problem under the same assumptions, the comparison isolates the contribution of ASCEND's central design choice: replacing the fixed full-background
conditioning set of CBL with dynamically updated nearest-ancestor sets.


We simulated two-tiered systems using the linear Gaussian process of Supplementary Section~S\ref{sec:simulation} and ran both methods on identical replicates. To stress the methods along the axes most relevant to omics applications, we performed one-dimensional sweeps over sample size, foreground
dimension, and background dimension around a default configuration (twelve configurations in total), evaluating five independent random seeds per setting;
exact values appear in Supplementary Section~S\ref{sec:simulation}. To keep the
comparison computationally tractable, we capped each run at a one-hour wall-clock budget; cells where CBL exceeded this budget are reported as conservative lower bounds. We report wall-clock runtime, the number of
conditional independence (CI) tests performed, coverage (the fraction of variable pairs the method commits a label to rather than leaving unresolved), and the F1 score against the true ancestral skeleton.

\subsubsection{Results}

Figure~\ref{fig:bench-cbl} summarises the comparison. ASCEND was significantly faster than CBL in every configuration tested, with speedup factors ranging from $27\times$ at the smallest problem
$(n{=}256,\,d_x{=}5,\,d_z{=}10)$ to over $5{,}000\times$ at $(n{=}1024,\,d_x{=}5,\,d_z{=}50)$. The runtime gap widens with the background dimension: empirical scaling exponents fitted along the
$d_z$ axis are $t \propto d_z^{+0.81}$ for CBL and $t \propto d_z^{-0.45}$ for ASCEND. ASCEND becomes \emph{faster} as the background grows because additional background variables refine the
nearest-ancestor sets and accelerate convergence; CBL becomes slower because the conditioning set widens by construction. At $d_z = 50$, a single CBL replicate did not finish within the one-hour
budget; ASCEND completed the same configuration in $0.27 \pm 0.01$~s.

The computational-work accounting tells the same story at a finer resolution. Across all twelve configurations, ASCEND extracted its output from between 10 and 4,560 conditional independence tests, whereas CBL performed
between 42,000 and 883,400. The factor of $10^2$--$10^3$ in raw conditional independence work, together with the lower per-test cost of ASCEND's F-test relative to CBL's $L_1$-regularised regression with $B=50$ subsamples, accounts
for the observed wall-clock gap.

ASCEND was also more decisive and more accurate. ASCEND resolved $99\% \pm 1\%$ of variable pairs on average, leaving almost none as $\texttt{NA}$, while CBL's stability-selection rule left between $7\%$
and $32\%$ of pairs unresolved (mean $26\%$). On F1 against the true ancestral skeleton, at the smallest problem $(d_x=5, d_z=10, n=1024)$, CBL identified zero true positives across all five seeds and
F1 is therefore undefined. 

\begin{figure*}[t]
\centering
\includegraphics[width=0.65\textwidth]{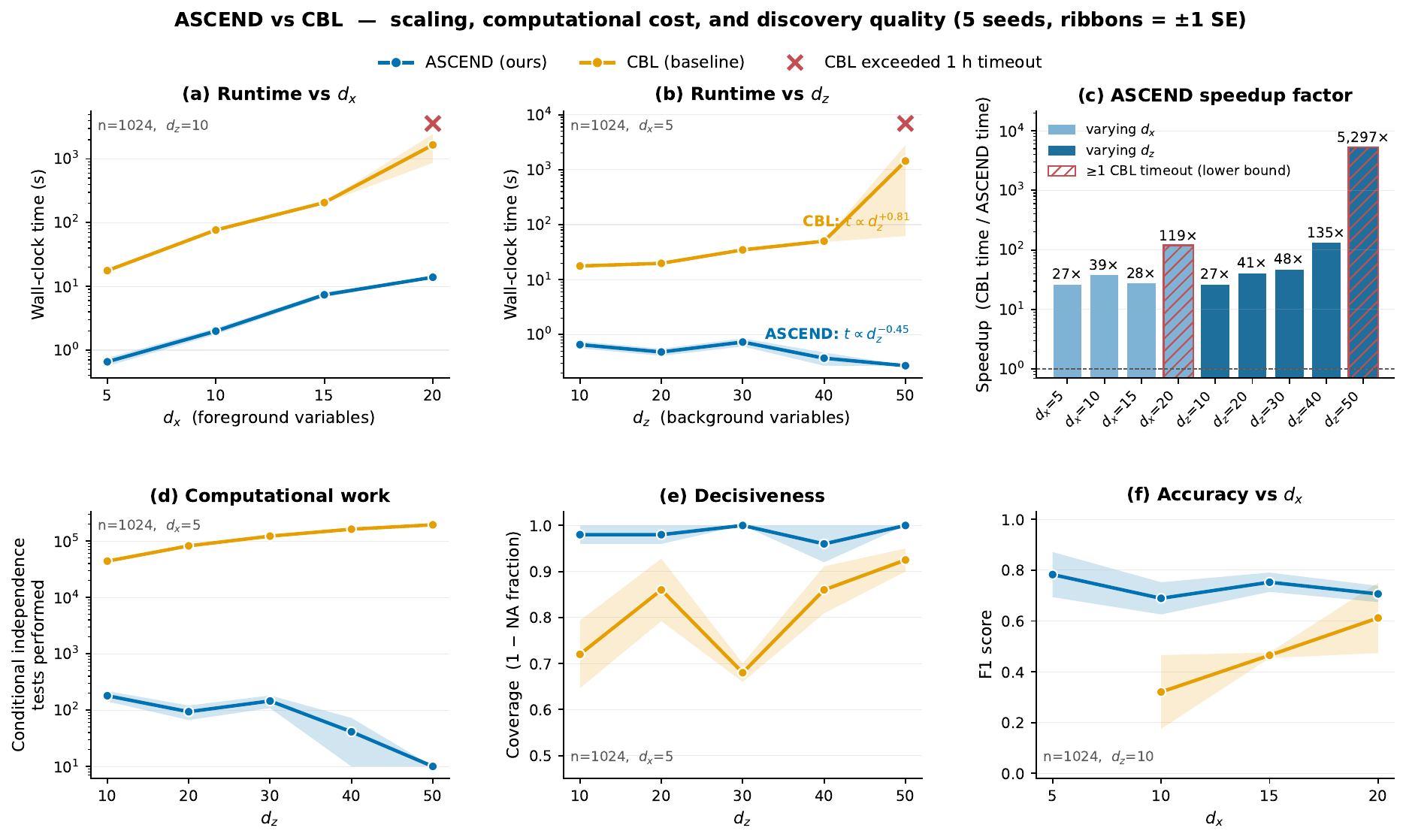}
\caption{\textbf{ASCEND vs. CBL runtime, computational cost, and discovery quality.} Ribbons show $\pm 1$ SEM. 
\textbf{(a--b)} Log wall-clock time with empirical scaling exponents. Red $\times$ marks CBL timeouts ($>1$ hour). 
\textbf{(c)} Speedup factor (CBL / ASCEND); red borders/labels indicate conservative lower bounds due to CBL timeouts. 
\textbf{(d)} Number of conditional independence tests performed. 
\textbf{(e)} Coverage (fraction of resolved variable pairs). 
\textbf{(f)} F1 score against the true ancestral skeleton (CBL is undefined at $d_x = 5$ due to zero true positives).}
\label{fig:bench-cbl}
\end{figure*}

\section{Comparison against causal discovery baselines}
\label{sec:bench-causal}

Having established ASCEND's advantage over standard GRN inference tools (Section~\ref{sec:bench-synthetic}), we now compare it against four causal discovery algorithms that explicitly target ancestral or directional structure: the confounder blanket learner (CBL)~\citep{watson2022causal}, the closest
methodological relative; Greedy Equivalence Search
(GES)~\citep{chickering2002optimal}; the linear non-Gaussian acyclic model (LiNGAM)~\citep{shimizu2006linear}; and the PC algorithm~\citep{spirtes2001causation}. All these methods, just like ASCEND, produces causal information instead of mere association.

\subsubsection{Experimental design}

Using the two-tiered data simulation of Supplementary
Section~S\ref{sec:simulation}, we tested a grid of 81 cells spanning a range of signal strengths ($R^2$), foreground dimensions, background-to-foreground
dimension ratios, and graph sparsities. For each cell we ran sample sizes ranging from 512 to 131{,}072 and up to 20 independent replicates per $(n,\text{cell})$. All methods received the same input matrix $(\mathbf{Z} \cup \mathbf{X})$ and were evaluated against the same true
ancestral skeleton derived from the simulator's DAG. Each method received a one-hour wall-clock budget per replicate. 

\subsubsection{Accuracy and the precision\textendash recall trade-off}

Figure~\ref{fig:bench-causal}a shows F1 against sample size for all five methods on the filtered subset. ASCEND attains the highest mean F1 at every sample size from $n = 512$ to $n = 131{,}072$. GES trails behind ASCEND closely; Figure~\ref{fig:bench-causal}b plots precision against $n$. ASCEND's precision is stable at $0.57 \pm 0.01$ across the well-sampled range of $n$, whereas GES rises monotonically from $0.43$ at $n = 512$ to $0.48$ at $n = 65{,}536$. This pattern reflects a known difference in calibration: ASCEND's constraint-based test is calibrated at all sample sizes, with its operating point set by the conditional independence test threshold ($\alpha = 0.05$) and largely insensitive to $n$. GES uses a Bayesian Information Criterion penalty whose ratio to the data likelihood depends on $\log(n)/n$, so as $n$ grows GES becomes increasingly conservative and asymptotically consistent. The crossing of the two methods' precision curves is therefore expected in the limit of very large $n$; within the sample-size range tested here, ASCEND remains higher than GES throughout. (The right-most data point at $n = 131{,}072$, drawn with open markers in panel b, is based on fewer than $50$ replicates per method and is more
indicative than conclusive.) CBL, LiNGAM and PC all remain below $0.40$ precision throughout.

On paired replicates against GES, ASCEND achieves higher precision on $67\%$ of $2{,}738$ paired runs, with a mean precision gain of $+0.091$ ($p < 10^{-30}$, Wilcoxon). This corresponds to ASCEND producing roughly half as many false positive ancestral edges as GES
per replicate ($16.1$ vs $30.4$ on average). ASCEND trades a small amount of recall for this precision: in paired replicates the mean recall gap is $-0.058$ in GES's favor. We argue that this trade-off favours ASCEND for the intended application of ancestral discovery in molecular networks. Each false-positive ancestral edge is a candidate hypothesis for experimental follow-up, and the cost of investigating a spurious edge dominates the cost of a
missed one in most realistic biological pipelines. Recall is a threshold-tunable property of the conditional-independence test; precision reflects the structural specificity of the method and is harder to recover after the fact.

To control for the differing failure profiles of the five methods, we compared paired F1 differences on the subset of replicates in which both methods succeeded (Figure~\ref{fig:bench-causal}c). Against CBL, LiNGAM and PC, the paired-difference distributions sit clearly to the
right of zero: mean gains of $+0.12$, $+0.07$ and $+0.51$ respectively, all significant at $p < 10^{-20}$ (Wilcoxon, two-sided). Against GES, the distribution straddles zero with a mean difference of $-0.03$; ASCEND and GES are statistically distinguishable at this sample size
(Wilcoxon $p < 10^{-30}$) but the effect is small and the distributions overlap substantially. Practically, ASCEND and GES deliver comparable F1, and the choice between them rests on the other properties measured in panels (b), (d) and (e).

ASCEND completed every replicate it was asked to run, failing in $0\%$ of attempted runs across the entire parameter grid (Figure~\ref{fig:bench-causal}e). LiNGAM's failure rate declined steadily as $n$ grew; CBL's failure rate dipped at intermediate $n$ before rising sharply again at the largest sample size tested. PC failed in $80$--$90\%$ of attempted runs across every $n$, and was included in our paired analysis only at cells where it produced sufficient replicates; this fragility is consistent with the known sensitivity of \texttt{pcalg}'s implementation to near-singular sample covariances at higher dimensions, and we therefore treat the ASCEND-versus-PC comparison as illustrative rather than definitive.

Among methods that produce ancestral output, only ASCEND and CBL emit explicit unresolved verdicts (\texttt{NA}) where the conditioning strategy fails to deliver an answer. GES, LiNGAM and PC always commit to a directed or absent edge; in the case of GES this is achieved by
taking the transitive closure of the inferred CPDAG, which can inflate apparent recall by extending the implications of partially oriented structure. Across the parameter grid, ASCEND committed to a label on $\ge 95\%$ of variable pairs at every sample size
(Figure~\ref{fig:bench-causal}d); CBL's coverage was lower ($75$--$83\%$ across most of the range, rising sharply to $\approx 97\%$ only at the largest sample size), reflecting its more conservative stability-selection rule.

\subsubsection{The biologically relevant regime}

Real gene-regulatory networks are believed to be
sparse~\citep{barabasi2004network}; the dense-graph regime that
challenges all causal-discovery methods is less representative of the
intended application. Restricting attention to the sparse, signal-rich
regime ( Figure~\ref{fig:bench-causal}f), ASCEND attains the highest F1 in every cell of the sparsity sweep, with a consistently large lead over GES across the sparsity levels tested
(F1 advantage of roughly $+0.12$ to $+0.13$, with no clear narrowing at higher sparsity). This pattern mirrors the GRN-baseline comparison of Section~\ref{sec:bench-synthetic}, where ASCEND's advantage was likewise substantial in the sparse regime: when true ancestral structure is sparse, constraint-based testing extracts more signal than score-based density estimation.

\begin{figure*}[t]
\centering
\includegraphics[width=0.85\textwidth]{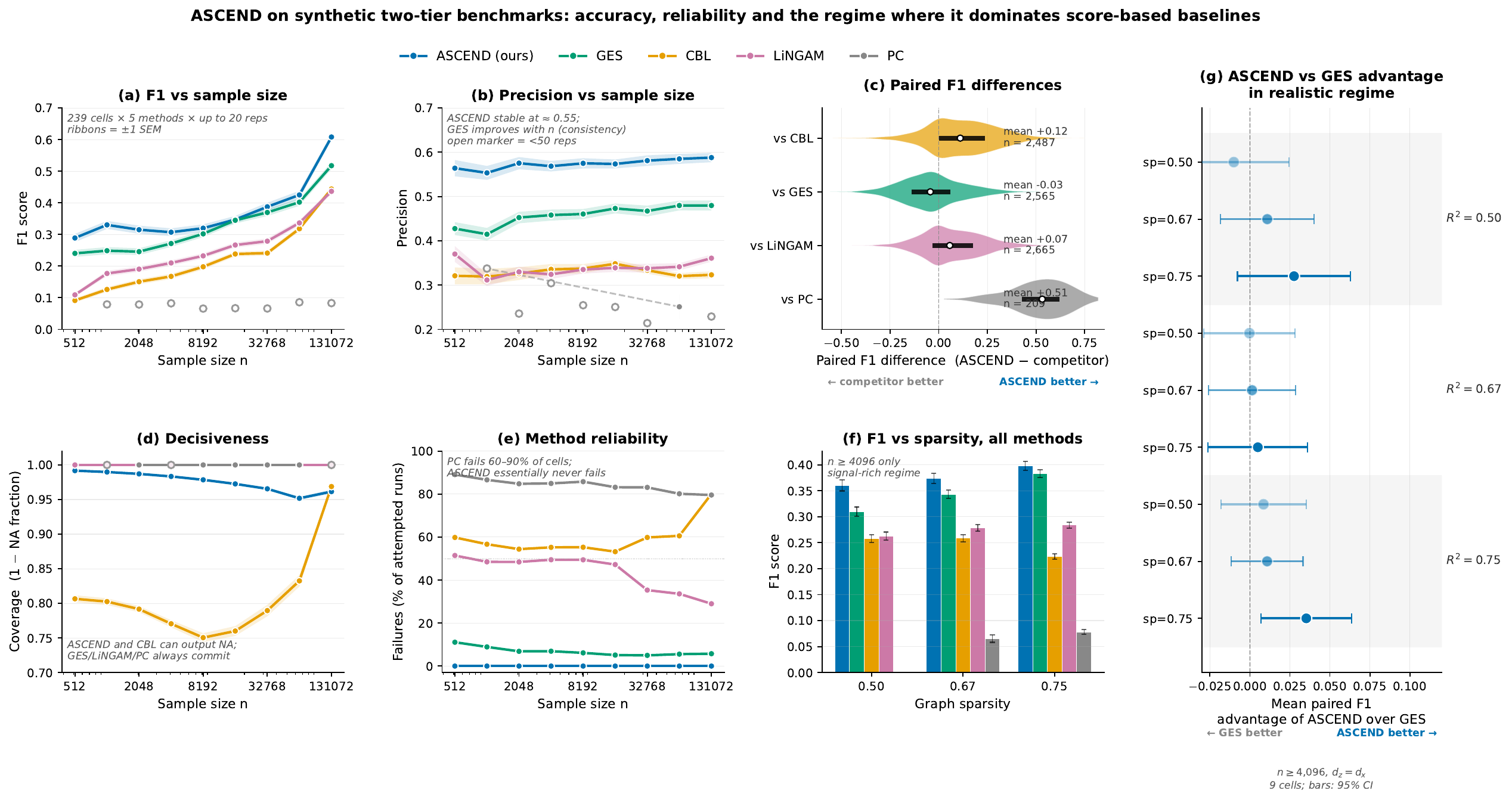}
\caption{\textbf{ASCEND vs. four causal discovery baselines on synthetic benchmarks.} Open markers/dashed lines indicate $<50$ replicates. 
\textbf{(a--b)} F1 score and Precision vs. sample size $n$. 
\textbf{(c)} Distribution of paired F1 differences (ASCEND $-$ competitor) on mutual replicates; white dots show medians, black bars span the IQR. 
\textbf{(d)} Coverage (fraction of foreground pairs labeled). 
\textbf{(e)} Method reliability across the 81-cell grid. 
\textbf{(f)} F1 vs. graph sparsity ($n \ge 4{,}096$). 
\textbf{(g)} Mean paired F1 difference (ASCEND $-$ GES) with 95\% CIs ($n \ge 4{,}096$, $d_z = d_x$) grouped by $R^2$ (darker colors indicate higher sparsity).}
\label{fig:bench-causal}
\end{figure*}

\section{Discussion}\label{sec4}

 We have introduced ASCEND, a constraint-based causal discovery framework that exploits a known two-tiered ordering between background and foreground variables to make ancestral discovery tractable at omics scale in polynomial time. By replacing the full-background conditioning
sets of CBL~\citep{watson2022causal} with dynamically updated nearest-ancestor sets, ASCEND localises each pairwise test to a small, biologically meaningful conditioning set, achieving substantial speedups, high decisiveness and precision; in the parameter regime
representative of real gene-regulatory networks.  In application, ASCEND identified
biologically coherent upstream regulators in the DGRP transcriptome. Beyond its standalone use, ASCEND's ancestral adjacency matrix can prune the candidate DAG space for any constraint-based causal discovery algorithm: The resulting ancestral matrix output can be used to define some relationships are constraints for algorithms like PC. ie ASCEND can also serve as a preprocessor. When 

\textbf{Limitations.} ASCEND's central assumption is a known causal ordering between two tiers of variables. Where this ordering is genuine (genotype $\to$ transcriptome, methylation $\to$ transcript) the assumption is a feature; where it is uncertain or absent, ASCEND is
not the right tool. The Markov-blanket oracle used in our implementation (IAMB with Fisher's Z and a fixed false-discovery rate of $0.05$) is well-calibrated at moderate dimensionality but can become conservative at very high foreground dimensionality with limited sample size. Finally, our synthetic benchmarks use linear-Gaussian structural equations; non-linear or non-Gaussian regulatory relationships will require alternative independence tests, and the score-based GES remains a competitive alternative in dense or weak-signal regimes where local constraint-based testing has less to exploit.

\subsubsection{Conclusion}
ASCEND closes the gap between the theoretical appeal of causal discovery and the practical realities of high-dimensional biological data. By exploiting structural priors that biology readily provides, it delivers ancestral inference at a scale and precision that make it usable as either a standalone discoverer or as a scalable preprocessor for any downstream causal pipeline.

\section{Data Access}
Transcriptomic and genomic data utilized in this study are available through the NCBI Gene Expression Omnibus (GEO) https://www.ncbi.nlm.nih.gov/geo/) under accession number $GSE117850$ and Zenodo (Record $14871341$), respectively. All preprocessed multi-omic matrices and the core analysis code for the ASCEND framework are available on GitHub at https://github.com/SparkAILab/Ascend.

\section{Competing Interests Declared}
The authors declare no competing interests.


\bibliographystyle{abbrvnat}
\bibliography{reference}

@article{Lynch2021The,title={The Meaning of "Cause" in Genetics.},author={Kate E. Lynch},journal={Cold Spring Harbor perspectives in medicine},year={2021},doi={10.1101/cshperspect.a040519}}

@article{chevalley2022causalbench,
  title={Causalbench: A large-scale benchmark for network inference from single-cell perturbation data},
  author={Chevalley, Mathieu and Roohani, Yusuf and Mehrjou, Arash and Leskovec, Jure and Schwab, Patrick},
  journal={arXiv preprint arXiv:2210.17283},
  year={2022}
}

@article{huynh2010inferring,
  title={Inferring regulatory networks from expression data using tree-based methods},
  author={Huynh-Thu, V{\^a}n Anh and Irrthum, Alexandre and Wehenkel, Louis and Geurts, Pierre},
  journal={PloS one},
  volume={5},
  number={9},
  pages={e12776},
  year={2010},
  publisher={Public Library of Science San Francisco, USA}
}

@book{spirtes2001causation,
  title={Causation, prediction, and search},
  author={Spirtes, Peter and Glymour, Clark and Scheines, Richard},
  year={2001},
  publisher={MIT press}
}

@inproceedings{Ayyanathan2014AssessingCR,
  title={Assessing Causal Relationships in Genomics: From Bradford-Hill Criteria to Complex Gene-Environment Interactions and Directed Acyclic Graphs},
  author={Kasirajan Ayyanathan},
  year={2014},
  url={https://api.semanticscholar.org/CorpusID:201870697}
}

@article{Hu2018ApplicationOC,
  title={Application of Causal Inference to Genomic Analysis: Advances in Methodology},
  author={Pengfei Hu and Rong Jiao and Li Jin and Momiao Xiong},
  journal={Frontiers in Genetics},
  year={2018},
  volume={9},
  url={https://api.semanticscholar.org/CorpusID:49652600}
}

@article{rebai2017causality,
  title={Causality in Genomics Studies: Time is ripe for a new Paradigm. Open J Bioinform Biostat 1 (1): 010-014},
  author={Rebai, A},
  journal={Life Sciences Group},
  year={2017}
}

@article{Danchin2007The,title={The extant core bacterial proteome is an archive of the origin of life},author={A. Danchin and Gang Fang and Stanislas Noria},journal={PROTEOMICS},year={2007},volume={7},doi={10.1002/PMIC.200600442}}

@article{Veenstra2012Metabolomics:,title={Metabolomics: the final frontier?},author={T. Veenstra},journal={Genome Medicine},year={2012},volume={4},pages={40 - 40},doi={10.1186/gm339}}

@article{Neale2019Gene,title={Gene Expression and the Transcriptome},author={D. Neale and N. Wheeler},journal={The Conifers: Genomes, Variation and Evolution},year={2019},doi={10.1007/978-3-319-46807-5_6}}

@article{Manel2016Genomic,title={Genomic resources and their influence on the detection of the signal of positive selection in genome scans},author={Stéphanie Manel and Charles Perrier and M. Pratlong and Laurent Abi-Rached and J. Paganini and Pierre Pontarotti and Didier Aurelle},journal={Molecular Ecology},year={2016},volume={25},doi={10.1111/mec.13468}}

@article{He2017Big,title={Big Data Analytics for Genomic Medicine},author={K. He and Dongliang Ge and Max M He},journal={International Journal of Molecular Sciences},year={2017},volume={18},doi={10.3390/ijms18020412}}

@article{Abu-Elmagd2022Editorial:,title={Editorial: Advances in genomic and genetic tools, and their applications for understanding embryonic development and human diseases},author={M. Abu-Elmagd and Mourad Assidi and A. Alrefaei and Ahmed Rebai},journal={Frontiers in Cell and Developmental Biology},year={2022},volume={10},doi={10.3389/fcell.2022.1016400}}

@article{Bates2020Causal,title={Causal inference in genetic trio studies},author={Stephen Bates and Matteo Sesia and C. Sabatti and E. Candès},journal={Proceedings of the National Academy of Sciences of the United States of America},year={2020},volume={117},pages={24117 - 24126},doi={10.1073/pnas.2007743117}}

@inproceedings{margolin2006aracne,
  title={ARACNE: an algorithm for the reconstruction of gene regulatory networks in a mammalian cellular context},
  author={Margolin, Adam A and Nemenman, Ilya and Basso, Katia and Wiggins, Chris and Stolovitzky, Gustavo and Favera, Riccardo Dalla and Califano, Andrea},
  booktitle={BMC bioinformatics},
  volume={7},
  pages={1--15},
  year={2006},
  organization={Springer}
}

@article{kalisch2007estimating,
  title={Estimating high-dimensional directed acyclic graphs with the PC-algorithm.},
  author={Kalisch, Markus and B{\"u}hlman, Peter},
  journal={Journal of Machine Learning Research},
  volume={8},
  number={3},
  year={2007}
}

@article{shimizu2006linear,
  title={A linear non-Gaussian acyclic model for causal discovery.},
  author={Shimizu, Shohei and Hoyer, Patrik O and Hyv{\"a}rinen, Aapo and Kerminen, Antti and Jordan, Michael},
  journal={Journal of Machine Learning Research},
  volume={7},
  number={10},
  year={2006}
}

@inproceedings{watson2022causal,
  title={Causal discovery under a confounder blanket},
  author={Watson, David S and Silva, Ricardo},
  booktitle={Uncertainty in Artificial Intelligence},
  pages={2096--2106},
  year={2022},
  organization={PMLR}
}

@article{chen2007harnessing,
  title={Harnessing naturally randomized transcription to infer regulatory relationships among genes},
  author={Chen, Lin S and Emmert-Streib, Frank and Storey, John D},
  journal={Genome biology},
  volume={8},
  pages={1--13},
  year={2007},
  publisher={Springer}
}

@article{glymour2019review,
  title={Review of causal discovery methods based on graphical models},
  author={Glymour, Clark and Zhang, Kun and Spirtes, Peter},
  journal={Frontiers in genetics},
  volume={10},
  pages={524},
  year={2019},
  publisher={Frontiers Media SA}
}

@article{magliacane2016ancestral,
  title={Ancestral causal inference},
  author={Magliacane, Sara and Claassen, Tom and Mooij, Joris M},
  journal={Advances in Neural Information Processing Systems},
  volume={29},
  year={2016}
}

@article{claassen2012logical,
  title={A logical characterization of constraint-based causal discovery},
  author={Claassen, Tom and Heskes, Tom},
  journal={arXiv preprint arXiv:1202.3711},
  year={2012}
}

@article{chickering2002optimal,
  title={Optimal structure identification with greedy search},
  author={Chickering, David Maxwell},
  journal={Journal of machine learning research},
  volume={3},
  number={Nov},
  pages={507--554},
  year={2002}
}

@inproceedings{entner2013data,
  title={Data-driven covariate selection for nonparametric estimation of causal effects},
  author={Entner, Doris and Hoyer, Patrik and Spirtes, Peter},
  booktitle={Artificial intelligence and statistics},
  pages={256--264},
  year={2013},
  organization={PMLR}
}

@inproceedings{tsamardinos2003algorithms,
  title={Algorithms for large scale Markov blanket discovery.},
  author={Tsamardinos, Ioannis and Aliferis, Constantin F and Statnikov, Alexander R and Statnikov, Er},
  booktitle={FLAIRS},
  volume={2},
  pages={376--81},
  year={2003}
}

@article{pratapa2020benchmarking,
  title={Benchmarking algorithms for gene regulatory network inference from single-cell transcriptomic data},
  author={Pratapa, Aditya and Jalihal, Amogh P and Law, Jeffrey N and Bharadwaj, Aditya and Murali, TM},
  journal={Nature methods},
  volume={17},
  number={2},
  pages={147--154},
  year={2020},
  publisher={Nature Publishing Group US New York}
}

@article{badia2023gene,
  title={Gene regulatory network inference in the era of single-cell multi-omics},
  author={Badia-i-Mompel, Pau and Wessels, Lorna and M{\"u}ller-Dott, Sophia and Trimbour, R{\'e}mi and Ramirez Flores, Ricardo O and Argelaguet, Ricard and Saez-Rodriguez, Julio},
  journal={Nature Reviews Genetics},
  volume={24},
  number={11},
  pages={739--754},
  year={2023},
  publisher={Nature Publishing Group UK London}
}

@article{dai2024gene,
  title={Gene regulatory network inference in the presence of dropouts: a causal view},
  author={Dai, Haoyue and Ng, Ignavier and Luo, Gongxu and Spirtes, Peter and Stojanov, Petar and Zhang, Kun},
  journal={arXiv preprint arXiv:2403.15500},
  year={2024}
}

@article{sandor2025efficient,
  title={Efficient structure learning of gene regulatory networks with Bayesian active learning},
  author={S{\'a}ndor, D{\'a}niel and Antal, P{\'e}ter},
  journal={BMC bioinformatics},
  volume={26},
  number={1},
  pages={150},
  year={2025},
  publisher={Springer}
}

@article{asiedu2025multi,
  title={Multi-omic Causal Discovery using Genotypes and Gene Expression},
  author={Asiedu, Stephen and Watson, David},
  journal={arXiv preprint arXiv:2505.15866},
  year={2025}
}

@inproceedings{shah2023causal,
  title={Causal Discovery and Optimal Experimental Design for Genome-Scale Biological Network Recovery},
  author={Shah, Ashka and Ramanathan, Arvind and Hayot-Sasson, Valerie and Stevens, Rick},
  booktitle={Proceedings of the Platform for Advanced Scientific Computing Conference},
  pages={1--11},
  year={2023}
}

@article{mao2025halo,
  title={HALO: hierarchical causal modeling for single cell multi-omics data},
  author={Mao, Haiyi and Jia, Minxue and Di, Marissa and Valenzi, Eleanor and Cai, Xiaoyu Tracy and Lafyatis, Robert and Zhang, Kun and Benos, Panayiotis V},
  journal={Nature Communications},
  volume={16},
  number={1},
  pages={8892},
  year={2025},
  publisher={Nature Publishing Group UK London}
}

@article{mackay2012drosophila,
  title={The Drosophila melanogaster genetic reference panel},
  author={Mackay, Trudy FC and Richards, Stephen and Stone, Eric A and Barbadilla, Antonio and Ayroles, Julien F and Zhu, Dianhui and Casillas, S{\`o}nia and Han, Yi and Magwire, Michael M and Cridland, Julie M and others},
  journal={Nature},
  volume={482},
  number={7384},
  pages={173--178},
  year={2012},
  publisher={Nature Publishing Group UK London}
}

@article{everett2020gene,
  title={Gene expression networks in the Drosophila genetic reference panel},
  author={Everett, Logan J and Huang, Wen and Zhou, Shanshan and Carbone, Mary Anna and Lyman, Richard F and Arya, Gunjan H and Geisz, Matthew S and Ma, Junwu and Morgante, Fabio and Armour, Genevieve St and others},
  journal={Genome research},
  volume={30},
  number={3},
  pages={485--496},
  year={2020},
  publisher={Cold Spring Harbor Lab}
}

@article{langfelder2008wgcna,
  title={WGCNA: an R package for weighted correlation network analysis},
  author={Langfelder, Peter and Horvath, Steve},
  journal={BMC bioinformatics},
  volume={9},
  number={1},
  pages={559},
  year={2008},
  publisher={Springer}
}

@article{barabasi2004network,
  title={Network biology: understanding the cell's functional organization},
  author={Barabasi, Albert-Laszlo and Oltvai, Zoltan N},
  journal={Nature reviews genetics},
  volume={5},
  number={2},
  pages={101--113},
  year={2004},
  publisher={Nature Publishing Group UK London}
}

@article{saito2015precision,
  title={The precision-recall plot is more informative than the ROC plot when evaluating binary classifiers on imbalanced datasets},
  author={Saito, Takaya and Rehmsmeier, Marc},
  journal={PloS one},
  volume={10},
  number={3},
  pages={e0118432},
  year={2015},
  publisher={Public Library of Science}
}

@book{Lehmann2005,
  author    = {Lehmann, E. L. and Romano, Joseph P.},
  title     = {Testing Statistical Hypotheses},
  edition   = {3rd},
  publisher = {Springer},
  address   = {New York},
  year      = {2005}
}

@inproceedings{bucur2018bayesian,
  title={A Bayesian Approach for Inferring Local Causal Structure in Gene Regulatory Networks},
  author={Bucur, Ioan Gabriel and Bussel, Tom and Claassen, Tom and Heskes, Tom},
  booktitle={International Conference on Probabilistic Graphical Models},
  pages={37--48},
  year={2018},
  organization={PMLR}
}

@article{teixeira2023yeastract,
  author = {Teixeira, M. C. and Viana, R. and Palma, M. and Oliveira, J. and Galocha, M. and Mota, M. N. and Couceiro, D. and Pereira, M. G. and Antunes, M. and Costa, I. V. and Pais, P. and Parada, C. and Chaouiya, C. and S{\'a}-Correia, I. and Monteiro, P. T.},
  title = {YEASTRACT+: a portal for the exploitation of global transcription regulation and metabolic model data in yeast biotechnology and pathogenesis},
  journal = {Nucleic Acids Research},
  volume = {51},
  number = {D1},
  pages = {D785--D791},
  year = {2023},
  doi = {10.1093/nar/gkac1041}
}

@article{lu2024suppression,
  title={Suppression of Drosophila antifungal immunity by a parasite effector via blocking GNBP3 and GNBP-like 3, the dual receptors for $\beta$-glucans},
  author={Lu, Mengting and Wei, Dongxiang and Shang, Junmei and Li, Shiqin and Song, Shuangxiu and Luo, Yujuan and Tang, Guirong and Wang, Chengshu},
  journal={Cell reports},
  volume={43},
  number={1},
  year={2024},
  publisher={Elsevier}
}

@article{perlmutter2024single,
  title={A single amino acid polymorphism in natural Metchnikowin alleles of Drosophila results in systemic immunity and life history tradeoffs},
  author={Perlmutter, Jessamyn I and Chapman, Joanne R and Wilkinson, Mason C and Nevarez-Saenz, Isaac and Unckless, Robert L},
  journal={Plos Genetics},
  volume={20},
  number={3},
  pages={e1011155},
  year={2024},
  publisher={Public Library of Science San Francisco, CA USA}
}

@article{clemmons2015effector,
  title={An effector peptide family required for Drosophila Toll-mediated immunity},
  author={Clemmons, Alexa W and Lindsay, Scott A and Wasserman, Steven A},
  journal={PLoS Pathogens},
  volume={11},
  number={4},
  pages={e1004876},
  year={2015},
  publisher={Public Library of Science San Francisco, CA USA}
}

@article{lindsay2018short,
  title={Short-form bomanins mediate humoral immunity in Drosophila},
  author={Lindsay, Scott A and Lin, Samuel JH and Wasserman, Steven A},
  journal={Journal of innate immunity},
  volume={10},
  number={4},
  pages={306--314},
  year={2018},
  publisher={S. Karger AG Basel, Switzerland}
}


\providecommand{\indep}{\perp\!\!\!\perp}
\providecommand{\nindep}{\not\!\perp\!\!\!\perp}
\providecommand{\dsep}{\perp_{d}}
\providecommand{\Pa}{\mathrm{Pa}}
\providecommand{\Ch}{\mathrm{Ch}}
\providecommand{\An}{\mathrm{An}}
\providecommand{\De}{\mathrm{De}}
\providecommand{\ND}{\mathrm{NonDe}}
\providecommand{\MB}{\mathrm{MB}}
\providecommand{\Sij}{\mathbf{S}_{ij}}
\providecommand{\Tcal}{\mathcal{T}}

\tikzset{
  node style/.style={circle, draw, minimum size=7mm, inner sep=0pt},
  arrow style/.style={-{Latex[length=2mm]}},
  bidir/.style={{Latex[length=2mm]}-{Latex[length=2mm]}}
}

\providecommand{\appendix}{\section*{Supplementary Material}\setcounter{section}{0}\renewcommand{\thesection}{\Alph{section}}}
\appendix

\section{ASCEND: Method Details}
\label{sec:supp-method}

ASCEND is a constraint-based causal-discovery framework for high-dimensional, two-tier
biological systems in which a set of background variables $\mathbf{Z}$ is known
to causally precede a set of foreground variables $\mathbf{X}$. ASCEND exploits
this ordering to replace the global structure-learning problem with a sequence of
local conditional-independence (CI) tests, each conditioned on a small,
dynamically maintained set of \emph{nearest ancestors}. This reduces both the
number and the order of CI tests relative to methods that condition on the full
background layer, while preserving soundness under standard assumptions.

\subsection{Notation and definitions}
\label{sec:supp-notation}

We encode causal structure as a directed acyclic graph (DAG)
$\mathcal{G}=(\mathbf{V},\mathbf{E})$ on observed variables
$\mathbf{V}=\mathbf{Z}\cup\mathbf{X}$, with $|\mathbf{Z}|=d_Z$,
$|\mathbf{X}|=d_X$. We use $\Pa(\cdot)$, $\Ch(\cdot)$, $\An(\cdot)$,
$\De(\cdot)$ for parents, children, ancestors and descendants respectively, and
$\ND(\cdot)=\mathbf{V}\setminus\De(\cdot)$ for non-descendants of a variable. A node $M$ is a \emph{mediator} of an $X$--$Y$ path if it is an internal non-collider on a
directed path $X\to\cdots\to Y$; a node $C$ is a \emph{collider} on a path if
both adjacent edges point into it. We write $A\dsep B\mid\mathbf{C}$ for
$d$-separation of $A$ and $B$ given $\mathbf{C}$ in $\mathcal{G}$, and
$A\indep B\mid\mathbf{C}$ for the corresponding conditional independence in the
data distribution $P$.

The \emph{ancestral relations} among foreground variables, recorded in a matrix
$\mathbf{A}\in\{0,0.5,1,\mathrm{NA}\}^{d_X\times d_X}$, are \\

$$A_{ij}=1 \;\Leftrightarrow\; X_i\prec X_j \;(X_i\in\An(X_j)),$$
$$A_{ij}=0.5 \;\Leftrightarrow\; X_i\preceq X_j \;(X_i\notin\De(X_j)),$$
$$A_{ij}=0 \;\Leftrightarrow\; X_i\sim X_j,$$
where $X_i\sim X_j$ means neither is an ancestor of the other, and
$\mathrm{NA}$ marks an undetermined pair. Ancestry is a strict partial order, i.e \textit{irreflexive}($X \prec X \vdash \text{FALSE}$), meaning a node cannot be its own ancestor; \textit{Asymmetric} ($X \prec Y \vdash Y \not\prec X$), meaning if $X$ is an ancestor of $Y$, then $Y$ cannot be an ancestor of $X$ and \textit{Transitive} ($X \prec Y \land Y \prec Z \vdash X \prec Z$), meaning if $X$ is an ancestor of $Y$ and $Y$ is an ancestor of $Z$, then $X$ is an ancestor of $Z$. $X\preceq Y$ is read ``$X$ is a
non-descendant of $Y$''.

\begin{definition}[Markov blanket and relativised parents]
\label{def:mb}
For $\Tcal\subseteq\ND(X)$, let $\Pa(X;\Tcal)$ denote the parent set of $X$ in
the latent projection of $\mathcal{G}$ onto $\Tcal\cup\{X\}$. Because $X$ has no
child in $\Tcal$ (children are descendants), this set equals the Markov blanket
of $X$ computed by CI tests restricted to $\Tcal$, and consists exactly of the
\emph{nearest ancestors} of $X$ within $\Tcal$.
\end{definition}

\begin{definition}[(De)activators]
\label{def:deact}
For a single node $W$ and conditioning set $\mathbf{C}$: $W$ is a
\emph{deactivator} of $(A,B\mid\mathbf{C})$, written
$A\indep B\mid\mathbf{C}\cup[W]$, if $A\nindep B\mid\mathbf{C}$ and
$A\indep B\mid\mathbf{C}\cup\{W\}$; $W$ is an \emph{activator}, written
$A\nindep B\mid\mathbf{C}\cup[W]$, if $A\indep B\mid\mathbf{C}$ and
$A\nindep B\mid\mathbf{C}\cup\{W\}$. \citep{claassen2012logical,watson2022causal}.
\end{definition}

\subsection{Assumptions}
\label{sec:supp-assumptions}

\begin{assumption}[Two-tier ordering]\label{a:tier}
No foreground variable is an ancestor of any background variable:
$\forall X_i\in\mathbf{X},\,Z_k\in\mathbf{Z},\ X_i\notin\An(Z_k)$. Biologically,
genotype or DNA methylation ($\mathbf{Z}$) precedes transcription ($\mathbf{X}$) and not the reverse.
\end{assumption}

\begin{assumption}[Markov and faithfulness]\label{a:faith}
$P$ is Markov and faithful to $\mathcal{G}$ \citep{spirtes2001causation}, this implies that every CI statement in $P$ coincides with a d-separation statement in $\mathcal{G}$ so $A\indep B\mid\mathbf{C}\Leftrightarrow A\dsep B\mid\mathbf{C}$.
\end{assumption}

The theoretical results in Appendix~\ref{sec:supp-theory} are stated for an
exact CI oracle $\mathcal{I}$; the finite-sample implementation replaces
$\mathcal{I}$ with partial correlation tests.

\subsection{The ASCEND algorithm}
\label{sec:supp-algo}

ASCEND maintains, for each $X$, a set $\Tcal_X^{(t)}$ of \emph{known
non-descendants} at iteration $t$, initialised at $\Tcal_X^{(1)}=\mathbf{Z}$
(valid by Assumption~\ref{a:tier}) and updated by
\begin{equation}\label{eq:Tupdate}
\Tcal_X^{(t+1)}=\{\,Y\in\mathbf{Z}\cup\mathbf{X} : Y\preceq_t X\,\},
\end{equation} where $\preceq_t$ is the partial order recorded after iteration $t$. Each iteration performs four steps.

\paragraph{Step 1 (Nearest ancestors).}
For each $X$ compute $\Pa(X;\Tcal_X^{(t)})$ by Markov-blanket discovery
restricted to $\Tcal_X^{(t)}$ (Definition~\ref{def:mb}); in practice we use
IAMB \citep{tsamardinos2003algorithms} with Fisher's $z$-test.

\paragraph{Step 2 (Conditioning set and CI test).}
For a pair $(X_i,X_j)$ we condition on the \emph{guarded} union of nearest
ancestors,
\begin{equation}\label{eq:guardedS}
\begin{split}
\Sij \;=\; \bigl\{\, \Pa(X_i;\Tcal_i)\cup \Pa(X_j;\Tcal_j) & \,\bigr\}\setminus\{ \De(X_i) \cup \De(X_j) \}
\end{split}
\end{equation}
keeping a foreground variable only if it is a known non-descendant of \emph{both}
endpoints; all background $\mathbf{Z}$ qualify automatically by
Assumption~\ref{a:tier}. We then test $X_i\indep X_j\mid\Sij$ by a
likelihood-ratio comparison of linear models, equivalent to the
partial-correlation test and uniformly most powerful under our assumptions
\citep{Lehmann2005}.

\begin{remark}[Why the both-sides condition is required]
\label{rem:guard}
The naive union $\Pa(X_i;\Tcal_i)\cup\Pa(X_j;\Tcal_j)$ is unsafe: a mediator
$M$ on $X_i\to M\to X_j$ is a non-descendant of $X_j$ yet a descendant of
$X_i$. Conditioning on $M$ blocks the only directed $X_i$--$X_j$ path, so R3
(below) would falsely declare $X_i\sim X_j$. The constraint
$\Sij\preceq X_i\wedge \Sij\preceq X_j$ in \eqref{eq:guardedS} excludes exactly such
mediators while retaining every common ancestor;
\end{remark}

\paragraph{Step 3 (Orientation rules).}
Building on definition \ref{def:deact}, for $W\in\Sij$ write $\mathbf{S}_{\setminus W}=\Sij\setminus\{W\}$. Building on
\citet{entner2013data}, \citet{magliacane2016ancestral} and
\citet{watson2022causal}:
\begin{itemize}
\item[(R1)] \emph{Deactivation.} If $\exists W: W\indep X_j\mid
  \mathbf{S}_{\setminus W}\cup[X_i]$, then $X_i\prec X_j$. 
\item[(R2)] \emph{Activation.} If $\exists W: W\nindep X_i\mid
  \mathbf{S}_{\setminus W}\cup[X_j]$, then $X_i\preceq X_j$.
\item[(R3)] \emph{Independence.} If $X_i\indep X_j\mid\Sij$, then
  $X_i\sim X_j$.
\end{itemize}

Minimal examples of structures identifiable by these rules are shown in Fig.~\ref{fig:dag-horizontal}. (a) depicts an instance of R1, (b) instance of R2 and (d) depicts an instance of R3. Fig \ref{fig:closure_symmetry} demonstrates examples of closure.

\paragraph{Step 4 (Closure and refinement).}
Applying transitivity ($A\prec B\prec C\Rightarrow A\prec C$) and symmetry ($A\preceq B\wedge B\preceq A\Rightarrow A\sim B$), update each
$\Tcal_X$ via equation \eqref{eq:Tupdate}, and recompute nearest ancestors. We repeat the iterations until no new relation is found. See algorithm \ref{alg:closure} for pseudocode.

\subsection{ASCEND Pseudocode (Oracle Version)}
\label{sec:supp-pseudo}

\begin{algorithm2e}[t]
\caption{\sc ASCEND - oracle version}
\label{alg:ascend}
\small\SetAlgoLined\DontPrintSemicolon
\KwIn{Background $\mathbf{Z}$, foreground $\mathbf{X}$, CI oracle $\mathcal{I}$}
\KwOut{Ancestrality matrix $\mathbf{M}$}
Initialise $\texttt{converged}\gets\texttt{FALSE}$, $\mathbf{M}\gets[\texttt{NA}]$,
$\Tcal_X^{(1)}\gets\mathbf{Z}\ \forall X$, $t\gets0$\;
\While{\textnormal{not} \texttt{converged}}{
  $t\gets t+1$;\quad $\texttt{converged}\gets\texttt{TRUE}$\;
  \For{each $X\in\mathbf{X}$}{$\Pa(X;\Tcal_X^{(t)})\gets\MB(X;\Tcal_X^{(t)})$\;}
  \For{each $X_i,X_j\in\mathbf{X}$ \textnormal{s.t.} $i>j$ \textnormal{and} $\mathbf{M}_{ij}=\textnormal{\texttt{NA}}$}{
    $\Sij \gets \Pa(X_i;\Tcal_i)\cup \Pa(X_j;\Tcal_j)  \,\bigr\}\setminus\{ \De(X_i) \cup \De(X_j)$
    \uIf{$\mathcal{I}(X_i\indep X_j\mid\Sij)$}{
      $\mathbf{M}_{ij}\gets i\sim j$;\quad $\texttt{converged}\gets\texttt{FALSE}$
    }\Else{
      \For{each $W\in\Sij$, $\mathbf{S}_{\setminus W}\gets\Sij\setminus\{W\}$}{
        \uIf{$\mathcal{I}(W\indep X_j\mid \mathbf{S}_{\setminus W}\cup[X_i])$}{
          $\mathbf{M}_{ij}\gets i\prec j$;\quad $\texttt{converged}\gets\texttt{FALSE}$
        }\uElseIf{$\mathcal{I}(W\indep X_i\mid \mathbf{S}_{\setminus W}\cup[X_j])$}{
          $\mathbf{M}_{ij}\gets j\prec i$;\quad $\texttt{converged}\gets\texttt{FALSE}$
        }\uElseIf{$\mathcal{I}(W\nindep X_j\mid \mathbf{S}_{\setminus W}\cup[X_i])$}{
          $\mathbf{M}_{ij}\gets \mathbf{M}_{ij}\wedge (j\preceq i)$;\quad $\texttt{converged}\gets\texttt{FALSE}$
        }\uElseIf{$\mathcal{I}(W\nindep X_i\mid \mathbf{S}_{\setminus W}\cup[X_j])$}{
          $\mathbf{M}_{ij}\gets \mathbf{M}_{ij}\wedge (i\preceq j)$;\quad $\texttt{converged}\gets\texttt{FALSE}$
        }
      }
    }
  }
  $\mathbf{M}\gets\textsc{Closure}(\mathbf{M})$
  \For{each $X\in\mathbf{X}$}{$\Tcal_X^{(t+1)}\gets\{Y\in\mathbf{Z}\cup\mathbf{X}: Y\preceq_t X\}$\;}
  \lIf{$\mathbf M^{(t)} = \mathbf M^{(t-1)}$} 
    {$\texttt{converged} \gets \texttt{TRUE}$}
}
\Return $\mathbf{M}$\;
\end{algorithm2e}

The bracket notation in Algorithm~\ref{alg:ascend} carries both conditions of
Definition~\ref{def:deact}; convergence is declared only when a full iteration
changes neither $\mathbf{M}$ nor any nearest-ancestor set.

\begin{algorithm2e}[h]
\caption{{\sc Closure}}
\label{alg:closure}
\small
\SetAlgoLined
\DontPrintSemicolon
\KwIn{Ancestrality matrix $\mathbf{M}$}
\KwOut{Updated ancestrality matrix $\mathbf{M}$}
\For{each $i, j \in \{1, \dots, d_X\}$ s.t. $i > j$}{
    \uIf{$(i \preceq_{\bf M} j \land i \succeq_{\bf M} j) \lor i \sim_{\bf M} j$}{
        $\mathbf{M}_{ij} \gets i \sim j$\;
    }
    \uElseIf{$i \prec_{\bf M} j$}{
        $\mathbf{M}_{ij} \gets i \prec j$\;
    }
    \uElseIf{$j \prec_{\bf M} i$}{
        $\mathbf{M}_{ij} \gets j \prec i$\;
    }
}
$\texttt{converged} \gets \texttt{FALSE}$\;
\While{not converged}{
    $\texttt{converged} \gets \texttt{TRUE}$\;
    \For{each $i, j, k \in \{1, \dots, d_X\}$ s.t. $i \neq j \neq k,~ i > k$}{
        \uIf{$i \prec_{\bf M} j \prec_{\bf M} k \land \mathbf{M}_{ik} \neq i \prec k$}{
            $\mathbf{M}_{ik} \gets i \prec k$,~ $\texttt{converged} \gets \texttt{FALSE}$\;
        }
        \uElseIf{$k \prec_{\bf M} j \prec_{\bf M} i \land \mathbf{M}_{ik} \neq k \prec i$}{
            $\mathbf{M}_{ik} \gets k \prec i$,~ $\texttt{converged} \gets \texttt{FALSE}$\;
        }
    }
}
\end{algorithm2e}

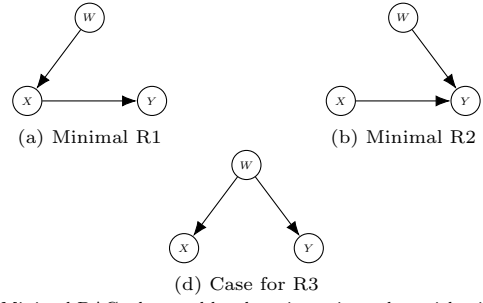
\begin{figure}[H]
  \centering
  \begin{minipage}{0.24\textwidth}\centering
    \begin{tikzpicture}[scale=0.55, transform shape]
      \node[node style] (Z) at (1.5,2) {$W$};
      \node[node style] (X) at (0,0) {$X$};
      \node[node style] (Y) at (3,0) {$Y$};
      \path[arrow style] (Z) edge (X) (X) edge (Y);
    \end{tikzpicture}\\ \small (a) Minimal R1
  \end{minipage}\hfill
  \begin{minipage}{0.24\textwidth}\centering
    \begin{tikzpicture}[scale=0.55, transform shape]
      \node[node style] (Z) at (1.5,2) {$W$};
      \node[node style] (X) at (0,0) {$X$};
      \node[node style] (Y) at (3,0) {$Y$};
      \path[arrow style] (Z) edge (Y) (X) edge (Y);
    \end{tikzpicture}\\ \small (b) Minimal R2
  \end{minipage}\hfill
  
  \begin{minipage}{0.24\textwidth}\centering
    \begin{tikzpicture}[scale=0.55, transform shape]
      \node[node style] (Z) at (1.5,2) {$W$};
      \node[node style] (X) at (0,0) {$X$};
      \node[node style] (Y) at (3,0) {$Y$};
      \path[arrow style] (Z) edge (X) (Z) edge (Y);
    \end{tikzpicture}\\ \small (d) Case for R3
  \end{minipage}
  \caption{Minimal DAGs detected by the orientation rules, with witness $W$.}
  \label{fig:dag-horizontal}
\end{figure}

\begin{figure}[H]
    \centering
    \includegraphics[width=0.7\columnwidth]{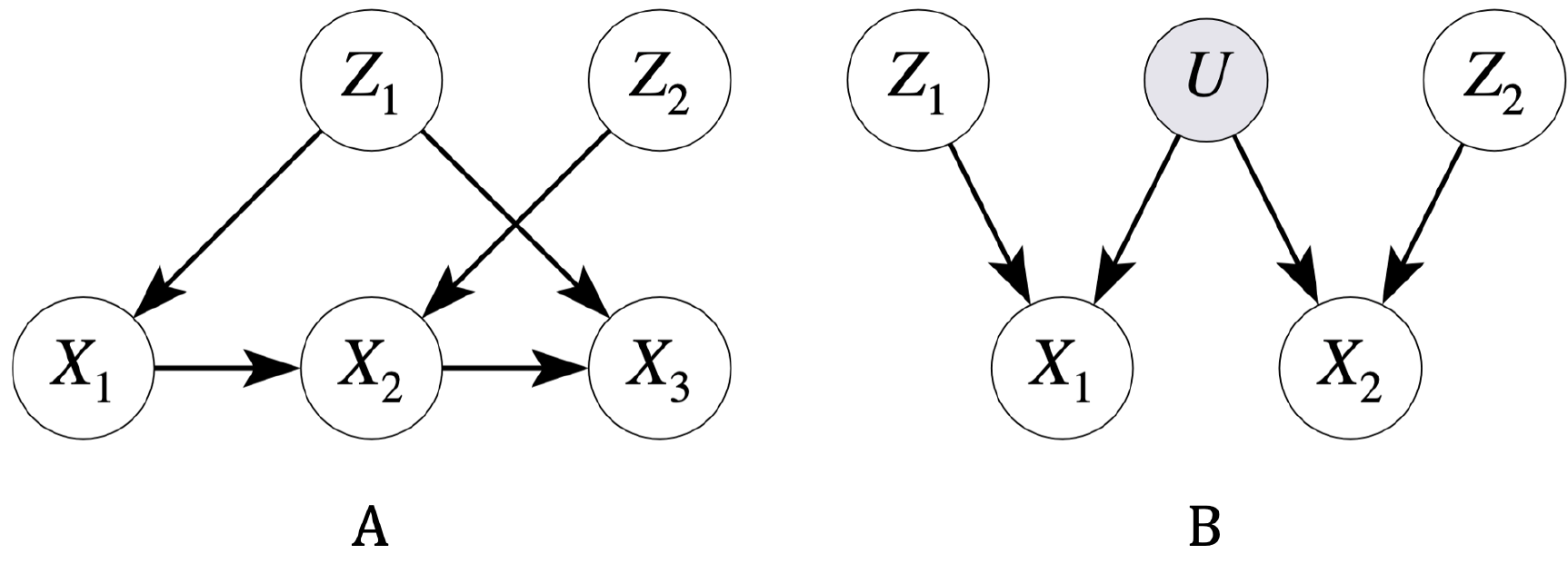}
    \caption{Example graphs illustrating how we exploits transitivity and antisymmetry to infer causal structure.}
    \label{fig:closure_symmetry}
\end{figure}

\section{Theoretical Guarantees}
\label{sec:supp-theory}

Throughout this section $\mathcal{I}$ is the exact oracle of
Assumption~\ref{a:faith}, so all independence statements are $d$-separation
statements. In this section we prove soundness, completeness and complexity.

\subsection{Soundness}
\label{sec:supp-sound}

\begin{lemma}[Non-descendant invariant]\label{lem:inv}
If every entry committed to $\mathbf{M}$ before sweep $t$ is correct, then
$\Tcal_X^{(t)}\subseteq\ND(X)$ for all $X$.
\end{lemma}
\begin{proof}
At $t=1$, $\Tcal_X^{(1)}=\mathbf{Z}\subseteq\ND(X)$ by Assumption~\ref{a:tier}, since $X$ is never an ancestor of a background variable.
For the step, $Y\in\Tcal_X^{(t+1)}$ iff $Y\preceq_t X$
\eqref{eq:Tupdate}; by hypothesis this relation is correct, so
$Y\notin\De(X)$.
\end{proof}

\begin{lemma}[Relativised parents are minimal and sufficient]\label{lem:paT}
Let $\Tcal\subseteq\ND(X)$, then $\Pa(X;\Tcal)$ equals the Markov blanket of $X$ computed by CI tests over $\Tcal\cup\{X\}$; it contains no descendant of $X$; and conditioning on it blocks every back-door path from $X$ whose intermediate nodes lie in $\Tcal$.
\end{lemma}
\begin{proof}
$X$ has no child in $\Tcal$, so its Markov blanket within $\Tcal\cup\{X\}$ is
the parent set of $X$ in the latent projection
(Definition~\ref{def:mb}). Marginalising a node outside $\Tcal\cup\{X\}$ adds an
adjacency to $X$ only along an inducing path that is into $X$, whose
$\Tcal$-endpoint is therefore an ancestor of $X$; marginalising
a collider child opens nothing. Hence every blanket member is an ancestor of $X$
in $\Tcal$, and projection parents block all represented back-door paths.
\end{proof}

\begin{definition}[Valid conditioning, VC]\label{def:vc}
A set $\mathbf{S}$ satisfies VC for $(X_i,X_j)$ if
$\mathbf{S}\subseteq\ND(X_i)\cap\ND(X_j)$.
\end{definition}

\begin{lemma}[The guarded set satisfies VC]\label{lem:vc}
If every committed entry of $\mathbf{M}$ is correct, then $\Sij$ of
\eqref{eq:guardedS} satisfies VC. Consequently no $W\in\Sij$ is a mediator of a
directed $X_i$--$X_j$ path.
\end{lemma}
\begin{proof}
Each foreground $W\in\Sij$ satisfies $W\preceq X_i$ and $W\preceq X_j$ by
construction, hence $W\in\ND(X_i)\cap\ND(X_j)$; each background
$W\in\mathbf{Z}$ lies in $\ND(X_i)\cap\ND(X_j)$ by Assumption~\ref{a:tier}. A
mediator of $X_i\to\cdots\to X_j$ lies in $\De(X_i)\cap\An(X_j)$ and so
violates $W\preceq X_i$; it is therefore excluded.
\end{proof}

\begin{theorem}[Soundness]\label{thm:sound}
Under Assumptions~\ref{a:tier} and \ref{a:faith} with the oracle $\mathcal{I}$,
every relation ASCEND commits is correct, and closure preserves
correctness.
\end{theorem}
\begin{proof}
We argue by induction over committed entries; Lemma~\ref{lem:vc} guarantees VC
at each commitment.

\emph{Base.} At $t=1$, $\Sij\subseteq\mathbf{Z}$; by Assumption~\ref{a:tier}
no $Z$ is a foreground mediator, so VC holds and any relation committed in the
first sweep is correct.

\emph{R1.} The deactivation $W\indep X_j\mid\mathbf{S}_{\setminus W}\cup[X_i]$
means adding $X_i$ turns $W,X_j$ from $d$-connected to $d$-separated. A node
only blocks paths on which it is a non-collider, so $X_i$ blocks every
$\mathbf{S}_{\setminus W}$-active $W$--$X_j$ path as a non-collider; under
$\Sij\preceq\{X_i,X_j\}$ this is the minimal-deactivation pattern of
\citet{entner2013data} and \citet{magliacane2016ancestral}, entailing
$X_i\in\An(X_j)$, i.e.\ $X_i\prec X_j$.

\emph{R2.} The activation $W\nindep X_i\mid\mathbf{S}_{\setminus W}\cup[X_j]$
means adding $X_j$ opens a $W$--$X_i$ path; a node opens a path only as (a
descendant of) a collider, so $X_j$ is such a collider, giving
$X_i\notin\De(X_j)$, i.e.\ $X_i\preceq X_j$.

\emph{R3.} Suppose $X_i\dsep X_j\mid\Sij$. By Lemma~\ref{lem:vc}, $\Sij$
contains no mediators of any directed $X_i$--$X_j$ path, so no directed path in
either direction is blocked, whence none exists and $X_i\sim X_j$.

\emph{Closure.} Ancestry $\prec$ is transitive and the symmetry rule is valid ($X_i \preceq X_j$ means $X_i \not\in De(X_j)$ and similarly $X_j \preceq X_i$ means $X_j \not\in De(X_i)$ so neither is an ancestor of the other, which is the definition of $X_i \sim X_j$); the acyclicity guard only ever rejects an update, so closure introduces no incorrect relation.
\end{proof}

\begin{remark}[Comparison with the unguarded union]
Dropping the both-sides condition in \eqref{eq:guardedS} breaks
Lemma~\ref{lem:vc}: a foreground mediator may enter $\Sij$ once oriented and
admitted to $\Tcal$, after which R3 can fire spuriously
(Remark~\ref{rem:guard}). The guard is what makes Theorem~\ref{thm:sound}
unconditional rather than holding only on the first sweep.
\end{remark}

\subsection{Completeness}
\label{sec:supp-complete}

Global completeness fails for any rule set of this form: when $X_i$ and $X_j$
are joined by parallel directed paths, no single (de)activation witness isolates
a mediator and the orientation is not identifiable, so $\mathrm{NA}$ is the
correct output. We adopt the \emph{lazy oracle} standard of \citet{watson2022causal}: we prove the two constructive statements that justify the localisation, then use them to show \textsc{ascend-oracle} is non-dominated within the lazy-oracle class, and is as
informative as \textsc{cbl-oracle}, but on a strictly smaller conditioning sets which makes ASCEND faster, cheaper and scalable.
The chain of how informative an inference is follows the following: $ \{\mathrm{NA}\}\ \prec\ \{i\preceq j\}\ \prec\ \{i\prec j\}\ \sim\ \{i\sim j\}  $. Based on this, we formulate the following.

\begin{definition}[Iteration-$t$ known non-descendant, \citealp{watson2022causal}]\label{def:knownnd}
$W$ is an \emph{iteration-$t$ known non-descendant} of $X$ if $W\in Z$, or if
$W\preceq_{\mathbf{M}}X$ after $t$ updates to $\mathbf{M}$. Write
$\mathbf{X}^{t}_{\preceq i}$ for this set for $X_i$ (so
$\mathbf{X}^{t}_{\preceq i}=\Tcal_{X_i}^{(t)}$).
\end{definition}

\begin{definition}[Lazy oracle algorithm, \citealp{watson2022causal}]\label{def:lazy}
A \emph{lazy oracle algorithm} starts from an uninformative $\mathbf{M}$ and
updates each round using only oracle answers to queries of two types:
(i) $W\indep X_i\mid\mathbf{S}_{ij}\setminus\{W\}\cup\phi(X_j)$, with
$W\in\mathbf{S}_{ij}$ and $\phi(X_j)\in\{\emptyset,\{X_j\}\}$; and
(ii) $X_i\indep X_j\mid\mathbf{S}_{ij}$. These are exactly the
(de)activation and separation queries of R1--R3, conditioned on the common known
non-descendants. 
\end{definition}

\begin{definition}[Dominance, \citealp{watson2022causal}]\label{def:dom}
$\mathcal{A}$ \emph{dominates} $\mathcal{B}$ iff (i) on no pair of any DAG is
$\mathcal{A}$'s output less informative than $\mathcal{B}$'s, and (ii) on some
pair of some DAG it is strictly more informative. $\mathcal{A}$ is
\emph{lazy-oracle complete} if no lazy oracle algorithm dominates it.
\end{definition}

\textsc{ascend-oracle} conditions on the guarded
nearest-ancestor set $\Sij\subseteq\mathbf{S}^{full}_{ij}$ of \eqref{eq:guardedS}
in place of the full $\mathbf{S}^{full}_{ij}$. In CBL's case, $\Sij$ includes all background variables and iteration-$t$ known non-descendants. 

\begin{lemma}[Localisation preserves separation]\label{lem:noloss}
Once $\Tcal_i,\Tcal_j$ contain all ancestors of $X_i,X_j$, the minimal guarded set
$\Sij$ and the full valid set $\mathbf{S}^{\mathrm{full}}_{ij}= Z \cup \big\{X \in \textbf{X} \backslash \{X_i, X_j\} : X \preceq_{\bf M} \{X_i, X_j\} \big\}$, as used in the confounder blanket learner \citet{watson2022causal},
give the same verdict: $X_i\dsep X_j\mid\Sij \Leftrightarrow X_i\dsep
X_j\mid\mathbf{S}^{\mathrm{full}}_{ij}$.
\end{lemma}

\begin{proof}
By Lemma~\ref{lem:paT}, $\Pa(X_i;\Tcal_i)$ is a Markov blanket of $X_i$ within
$\mathbf{S}^{full}_{ij}$ (and likewise for $X_j$), so every member of
$\mathbf{S}^{full}_{ij}\setminus\Sij$ is screened from both endpoints by elements of
$\Sij$. For the separation query, $\Sij$ already blocks all back-door paths; the
extra members are redundant non-colliders, or colliders re-blocked by a parent in
$\Sij$, so $X_i\dsep X_j\mid\Sij\Leftrightarrow X_i\dsep X_j\mid\mathbf{S}^{full}_{ij}$.
For a (de)activation query, a member of $\mathbf{S}^{t}_{ij}\setminus\Sij$ is
screened from the target and cannot toggle the dependence, so the before/after
verdict is unchanged; and a witness outside $\Sij$ is a non-nearest ancestor,
already independent of the target given the remaining blanket, so it fails the
``before'' dependence and is never a minimal (de)activator for both oracles. 
\end{proof}

\begin{theorem}\label{thm:sep}
Under Assumptions~\ref{a:tier} and \ref{a:faith} with the oracle, if
$X_i\sim X_j$ and the pair is separable, then \textsc{ascend-oracle} commits
$A_{ij}=0$ at convergence; a confounded $\sim$ pair is instead recovered when
$X_i\preceq X_j$ and $X_j\preceq X_i$ are both inferred, which closure collapses
to $X_i\sim X_j$.
\end{theorem}
\begin{proof}
If $X_i\sim X_j$ no directed path joins them, and the ancestors that separate
them lie in $\ND(X_i)\cap\ND(X_j)$. Background such ancestors are in
$\Tcal^{(1)}$; each foreground common ancestor $X_m$ satisfies $X_m\prec X_i$ and
$X_m\prec X_j$ and, by the depth induction of Theorem~\ref{thm:orient}, is
oriented and enters both $\Tcal_i,\Tcal_j$. Once all are represented, $\Sij$
blocks every back-door path while blocking no directed path, so
$X_i\dsep X_j\mid\Sij$ and R3 fires. When no separating set exists but both
$X_i\preceq X_j$ and $X_j\preceq X_i$ are inferred, neither is an ancestor of the
other, so closure records $X_i\sim X_j$.
\end{proof}

\begin{theorem}[Orientation-completeness relative to the rule basis]
\label{thm:orient}
Under Assumptions~\ref{a:tier} and \ref{a:faith} with the oracle, at convergence
ASCEND commits the correct relation for every rule-identifiable pair whose
(de)activating witness lies among the nearest ancestors; in particular every
foreground parent--child relation is oriented.
\end{theorem}
\begin{proof}
Define the ancestral depth $d(X)$ as the longest directed foreground path into
$X$. We show by induction on a depth of $k$ that after finitely many
sweeps $\Pa(X;\Tcal_X)$ equals the true nearest-ancestor set of every $X$ with
$d(X)\le k$.

\emph{Base} ($k=0$): a source has only background ancestors, captured at $t=1$.

\emph{Step}: the foreground parents of a depth-$k$ node have depth $<k$, so by
hypothesis they are correct and, being parents, are oriented by R1; once
oriented they satisfy as non-descendants of $X$ and enter $\Tcal_X$ via \eqref{eq:Tupdate},
so the next blanket recomputation returns them. At the resulting fixed point,
$\Sij$ matches $\mathbf{S}^{\mathrm{full}}_{ij}$ in separation content
(Lemma~\ref{lem:noloss}) and contains every nearest-ancestor witness, so every
rule application available to the full-conditioning oracle is available to
ASCEND and fires the same relation. Convergence is reached only when no further
rule can fire.
\end{proof}

\begin{theorem}[Lazy-oracle completeness]\label{thm:complete}
\textsc{ascend-oracle} is lazy-oracle complete: no lazy oracle algorithm
dominates it. Equivalently, it is exactly as informative as \textsc{cbl-oracle},
attained with a minimal conditioning set making it cheaper, faster and scalable.
\end{theorem}
\begin{proof}
By Theorems~\ref{thm:sep} and \ref{thm:orient}, \textsc{ascend-oracle} resolves
every pair that R1--R3 and closure can identify from the non-descendants of both
endpoints. By Lemma~\ref{lem:noloss} its guarded queries return the same answers
as the full-set queries of Definition~\ref{def:lazy}, so it makes exactly the
inferences \textsc{cbl-oracle} makes; the two are information-equivalent. Since
no lazy oracle algorithm dominates \textsc{cbl-oracle}
\citep[Thm.~3]{watson2022causal}, none dominates \textsc{ascend-oracle}. The
size bound follows from $\Sij\subseteq\mathbf{S}^{full}_{ij}$ with $\Sij$ the
nearest-ancestor blanket. Since $\Sij$ is smaller, the \textsc{ASCEND oracle} is called fewer times, making it scalable for omics unlike $CBL$.
\end{proof}

\begin{remark}[Scope and advantage]
Theorem~\ref{thm:complete} is completeness against the lazy oracle, not absolute
identifiability. What separates \textsc{ascend-oracle} from
\textsc{cbl-oracle} is not the inferences but their cost: it reaches the same
lazy-oracle-optimal output conditioning only on nearest ancestors, so $|\Sij|$
stays small (and shrinks as $d_Z$ grows), whereas $|\mathbf{S}^{full}_{ij}|$ grows
with the resolved ancestry and with $d_Z$.
\end{remark}

\subsection{Complexity}
\label{sec:supp-complexity}

\begin{theorem}[Complexity]\label{thm:complexity}
Let $\bar s$ be the maximum size of any conditioning set encountered. \textsc{ASCEND oracle} terminates in $O\!\bigl(d_X^{2}(d_X+d_Z)\,\bar s\bigr)$ in the worst
case.
\end{theorem}
\begin{proof}
Per pair per sweep: one R3 test plus $O(|\Sij|)=O(\bar s)$ (de)activation
tests; maintaining the two blankets costs $O((d_X+d_Z)\bar s)$ oracle queries,
absorbed over the pairs that share a node. There are
$\binom{d_X}{2}=O(d_X^2)$ pairs. The committed relations form a partial order
of height $\le d_X$, and each sweep resolves at least one further depth layer
(Theorem~\ref{thm:orient}), so $O(d_X)$ sweeps suffice. Multiplying gives
$O(d_X^2(d_X+d_Z)\bar s)$; in sparse graphs the nearest-ancestor sets do not
grow with $d_Z$, $\bar s=O(1)$, and the bound is $O(d_X^2\bar s)$.
\end{proof}

\begin{remark}
The realised cost scales with $\bar s$, which \emph{shrinks} as additional
background variables refine the nearest-ancestor sets, whereas full-background
conditioning has per-test cost growing in
$d_Z$. This is the source of the empirical runtime advantage reported in
Section~\ref{sec:bench-cbl}.
\end{remark}

\section{Implementation Details}
\label{sec:supp-impl}

Markov blankets are learned by IAMB with a Fisher-$z$ partial-correlation test
at level $\alpha_{\mathrm{mb}}=0.05$; Pairwise tests use the same statistic, reporting $q$-values via the Benjamini–Hochberg correction across each sweep. 
All CI queries read from a single correlation matrix computed once, so each test is an $O(|\Sij|^3)$
precision-matrix solve rather than a model refit. In finite samples we aggregate
(de)activation evidence across all witnesses $W\in\Sij$ and commit the direction
with the larger summed $-\log p$, which improves robustness without altering the
oracle guarantees (the oracle commits on the first valid witness).

\section{Simulation Framework}
\label{sec:simulation}

We generate synthetic two-tier systems from a linear-Gaussian structural
equation model (SEM) respecting the hierarchy of Assumption~\ref{a:tier}. With
$\mathbf{Z}=\{Z_1,\dots,Z_{d_Z}\}$ and $\mathbf{X}=\{X_1,\dots,X_{d_X}\}$,
the joint DAG admits edges $Z_i\to Z_j$, $Z_i\to X_j$, $X_i\to X_j$, and
forbids $X_i\to Z_j$. Sampling in topological order,
\begin{equation}
\begin{aligned}
Z_j &= \sum_{Z_i\in\Pa_{\mathbf{Z}}(Z_j)}\beta_{ij}Z_i+\epsilon_j, \\
    &\quad \epsilon_j \sim\mathcal{N}(0,\sigma_j^2), \\
X_k &= \sum_{Z_i\in\Pa_{\mathbf{Z}}(X_k)}\gamma_{ik}Z_i 
       + \sum_{X_i\in\Pa_{\mathbf{X}}(X_k)}\omega_{ik}X_i+\delta_k, \\
    &\quad \delta_k \sim\mathcal{N}(0,\psi_k^2),
\end{aligned}
\end{equation}
where $\beta,\gamma,\omega$ are edge coefficients and the noise variances are
scaled to a target signal-to-noise ratio $R^2$ (the fraction of each variable's
variance explained by its parents). Table~\ref{tab:params} lists the swept
parameters.

\begin{table}[H]
\centering
\caption{Simulation parameters.}
\label{tab:params}
\small
\begin{tabular}{p{0.28\linewidth} p{0.62\linewidth}}
\toprule
\textbf{Parameter} & \textbf{Interpretation} \\
\midrule
Sample size $n$ & Statistical power for CI testing; low $n$ stresses
  data-scarce omics regimes. \\
Background dim.\ $d_Z$ & Number of background variables. \\
Foreground dim.\ $d_X$ & Number of foreground variables. \\
Sparsity $sp$ & Ratio of present to possible edges (larger
  $\Rightarrow$ sparser). \\
Signal strength $R^2$ & Variance explained by causal parents; lower $R^2$
  tests weak-effect robustness. \\
\bottomrule
\end{tabular}
\end{table}

\paragraph{Ground truth and evaluation.}
From $\mathcal{G}$ we extract the foreground ancestral matrix
$\mathbf{A}^{\mathbf{X}}$ with $\mathbf{A}^{\mathbf{X}}_{ij}=1$ iff
$X_i\in\An(X_j)$. Because the estimand is a \emph{directed} ancestral relation,
we score all ordered pairs $(i,j)$, $i\ne j$: a cell is a positive prediction
when $A_{ij}\in\{0.5,1\}$ and a true positive when additionally
$\mathbf{A}^{\mathbf{X}}_{ij}=1$. We report precision, recall and $F_1$ over
these directed claims, the fraction of pairs ASCEND resolves (coverage), and
direction accuracy among truly adjacent, oriented pairs. Pairs left
$\mathrm{NA}$ are reported separately and are not counted as errors, reflecting
ASCEND's design preference for an honest abstention over a forced orientation.

\section{Implementation details for the causal-discovery baseline comparison}
\label{sec:supp-bench-causal-impl}

This section documents settings specific to the comparison against CBL, GES,
LiNGAM and PC (Section~\ref{sec:bench-causal}); general data-generation
mechanics are described in Section~\ref{sec:simulation}.

\paragraph{Fixed generative settings.}
The parameters swept for this comparison are $n$, $d_X$, $d_Z/d_X$, $sp$ and
$R^2$ (Table~\ref{tab:params}). Two further SEM parameters were held fixed
across the entire grid rather than swept: the background$\to$foreground edge
probability $P(Z_i\to X_j)=0.20$, and the foreground$\to$foreground effect
size ($X\to X$ coefficient magnitude) $=0.90$. All parent contributions were
fully linear ($\texttt{lin\_pr}=1$).

\paragraph{Method configurations.}
All competitor implementations are from \texttt{pcalg} unless noted.
\begin{itemize}
\item \textbf{ASCEND}: $\texttt{maxiter}=10$, $\alpha=0.05$ for pairwise CI
  tests, $\alpha_{\mathrm{mb}}=0.05$ for Markov-blanket discovery.
  
\item \textbf{CBL}: stability-selection with $\gamma=0.5$, $\texttt{maxiter}=10$,
  and $B=20$ bootstrap resamples per call.
\item \textbf{GES}: \texttt{pcalg::ges} with a \texttt{GaussL0penObsScore} and
  \texttt{iterate=TRUE}, run on the full $\mathbf{Z}\cup\mathbf{X}$ matrix.
\item \textbf{LiNGAM}: \texttt{pcalg::lingam} on the full matrix, using the
  pruned coefficient matrix (\texttt{Bpruned}) when available.
\item \textbf{PC}: \texttt{pcalg::pc} with \texttt{gaussCItest}, $\alpha=0.05$,
  \texttt{maj.rule=TRUE}, \texttt{solve.confl=TRUE}, on a sufficient statistic
  built from the complete-case, non-constant columns of the full matrix.
\end{itemize}

\paragraph{Compute environment and job layout.}
Runs were performed on a SLURM cluster (KCL CREATE) on R~4.3.1. Each SLURM array
task corresponds to one (parameter combination, $n$) pair and runs all $20$
replicates for that pair, executing all five methods serially within the
task (one CPU per task); the $81$ combinations $\times$ $9$ sample sizes give
$729$ tasks in total. Wall-time and memory were allocated per $n$-band
(4--36 hours, 4--32\,GB) rather than uniformly, since small-$n$ tasks are
inexpensive and $n=131{,}072$ tasks are not.

\paragraph{Timeouts and failure propagation.}
Each method was allotted a hard $3{,}600$-second (1-hour) budget per
replicate, enforced with \texttt{R.utils::withTimeout}. If a method timed out
at a given $(\text{combination}, n)$, it was marked \texttt{skipped\_timeout}
and skipped for the remaining replicates at that cell \emph{and at every
larger $n$ for the same combination} (persisted via a marker file, so later,
larger-$n$ tasks inherit the skip). This is why a method's reported failure
rate at large $n$ reflects both genuine timeouts at that $n$ and propagated
skips from smaller $n$ within the same combination, and should not be read as
$729$ independent failure draws. Run status is recorded per replicate as one
of \texttt{ok}, \texttt{timeout}, \texttt{error}, \texttt{skipped\_timeout} or
\texttt{failed\_subset}.

\paragraph{Evaluation bookkeeping.}
Metrics are computed once per replicate over the upper triangle of ordered
foreground pairs. Precision is $tp/(tp+fp)$ and left \texttt{NA} when the
method makes no positive call; recall is
$tp/(tp+fn+\texttt{unres\_tp})$, i.e.\ an unresolved (\texttt{NA}) verdict on
a truly ancestral pair counts against recall even though it is not scored as
a false negative outright, while an unresolved verdict on a truly-unrelated
pair (\texttt{unres\_tn}) is excluded from both the precision and recall
denominators. \texttt{NA} precision/recall/F1 values are preserved as
\texttt{NA} in the merged output (not imputed to 0), so mean statistics
reported in Section~\ref{sec:bench-causal} are means over the replicates on
which a method actually committed to a value.

\section{Drosophila Causal Hub Discovery}
\label{supp:dgrp_hubs}

This section provides the full network statistics, extended regulator table, and complete edge lists underlying the DGRP causal hub analysis reported in the main text (the DGRP section).

\subsection{Network Summary Statistics}

Applying ASCEND to the 250 transcripts yielded a sparse, structured network (Table~\ref{tab:network_summary}), consistent with a small number of high-leverage regulatory loci.

\begin{table}[h]
\centering
\caption{Summary statistics of the ASCEND-inferred causal network on the DGRP transcriptome.}
\label{tab:network_summary}
\begin{tabular}{l c}
\hline
Statistic & Value \\
\hline
Transcripts analysed & 250 \\
Directed (oriented) edges & 880 \\
Undirected (unoriented) edges & 61 \\
Total edges & 941 \\
Network density & 3.02\% \\
Genes with out-degree $>0$ & 121 / 250 (48.4\%) \\
Genes with out-degree $\geq 5$ & 50 / 250 (20.0\%) \\
\hline
\end{tabular}
\end{table}

\begin{table}[h]
\centering
\caption{Chromosomal distribution of ASCEND-discovered regulatory hubs among the 250 selected transcripts.}
\label{tab:chrom_dist}
\begin{tabular}{l c c c}
\hline
Chromosome & Mapped genes analysed & Hub genes & \% hubs \\
\hline
X & 20 & 10 & 50.0\% \\
2L & 46 & 16 & 34.8\% \\
2R & 68 & 38 & 55.9\% \\
3L & 45 & 20 & 44.4\% \\
3R & 66 & 34 & 51.5\% \\
\hline
\textbf{Total} & 245 & 118 & 48.2\% \\
\hline
\end{tabular}
\end{table}

\subsection{Extended Regulator Table}

Table~\ref{tab:extended_hubs} lists all 50 genes with causal out-degree $\geq 5$, extending the top-7 summary in the main text. Seven of the top 20 out-degree genes are recognised immune effectors or pattern-recognition components, a stronger version of the immune-module recovery argument than the three-gene example given in the main text.

\begin{table*}[h]
\centering
\caption{Extended list of ASCEND-discovered regulatory hubs (out-degree $\geq$ 5, $n=50$ genes).}
\label{tab:extended_hubs}
\small
\begin{tabular}{l l c c c l l}
\hline
Generic ID & Symbol & Out-deg. & In-deg. & Total & Chr. & FlyBase ID \\
\hline
x170 & GNBP-like3 & 51 & 0 & 51 & 2R & FBgn0034511 \\
x245 & CG10911 & 49 & 0 & 49 & 2R & FBgn0034295 \\
x86 & Mtk & 48 & 2 & 50 & 2R & FBgn0014865 \\
x98 & Dso2 & 33 & 3 & 36 & 2R & FBgn0067905 \\
x149 & \textit{unannotated} & 31 & 3 & 34 & - & - \\
x160 & ninaD & 29 & 0 & 29 & 2L & FBgn0002939 \\
x176 & mag & 28 & 1 & 29 & 3L & FBgn0036996 \\
x28 & BomBc1 & 27 & 4 & 31 & 2R & FBgn0034328 \\
x229 & BomBc3 & 25 & 6 & 31 & 3R & FBgn0040582 \\
x20 & PGRP-SB1 & 25 & 3 & 28 & 3L & FBgn0043578 \\
x235 & CG13323 & 20 & 2 & 22 & 2R & FBgn0033788 \\
x148 & PGRP-SD & 18 & 8 & 26 & 3L & FBgn0035806 \\
x115 & rdhB & 18 & 0 & 18 & 3R & FBgn0038946 \\
x171 & CG15818 & 18 & 0 & 18 & 2L & FBgn0031910 \\
x26 & lncRNA:CR32368 & 17 & 14 & 31 & 3L & FBgn0052368 \\
x31 & DptA & 17 & 4 & 21 & 2R & FBgn0004240 \\
x57 & DptB & 16 & 5 & 21 & 2R & FBgn0034407 \\
x49 & AttC & 15 & 17 & 32 & 2R & FBgn0041579 \\
x110 & CG8661 & 15 & 0 & 15 & X & FBgn0030837 \\
x237 & Jheh1 & 15 & 0 & 15 & 2R & FBgn0010053 \\
x166 & Sodh2 & 14 & 1 & 15 & 3R & FBgn0022359 \\
x25 & CG6277 & 13 & 20 & 33 & 3R & FBgn0039475 \\
x136 & CG42397 & 12 & 6 & 18 & 3L & FBgn0259748 \\
x242 & Cyp6d5 & 12 & 1 & 13 & 3R & FBgn0038194 \\
x213 & Jon99Cii & 12 & 0 & 12 & 3R & FBgn0003356 \\
x140 & CG5767 & 11 & 22 & 33 & 2R & FBgn0034292 \\
x209 & CG15534 & 10 & 0 & 10 & 3R & FBgn0039769 \\
x52 & Jon99Ci & 9 & 2 & 11 & 3R & FBgn0003358 \\
x156 & Cyp4d14 & 8 & 2 & 10 & X & FBgn0023541 \\
x196 & CG42650 & 8 & 0 & 8 & 3R & FBgn0261502 \\
x220 & tobi & 8 & 0 & 8 & 3R & FBgn0261575 \\
x117 & Jon25Bi & 8 & 0 & 8 & 2L & FBgn0020906 \\
x18 & Cyp6a8 & 7 & 3 & 10 & 2R & FBgn0013772 \\
x134 & CG9394 & 7 & 0 & 7 & 2R & FBgn0034588 \\
x183 & CG17380 & 6 & 5 & 11 & 3R & FBgn0039077 \\
x70 & CG30487 & 6 & 4 & 10 & 2R & FBgn0050487 \\
x131 & CG12057 & 6 & 1 & 7 & X & FBgn0030098 \\
x36 & CG43055 & 6 & 0 & 6 & 2L & FBgn0262357 \\
x240 & CG42710 & 6 & 0 & 6 & 2L & FBgn0261627 \\
x92 & CG18180 & 6 & 0 & 6 & 3L & FBgn0036024 \\
x87 & CG13405 & 5 & 7 & 12 & 3L & FBgn0035097 \\
x143 & BomS1 & 5 & 7 & 12 & 2R & FBgn0034329 \\
x79 & CG31226 & 5 & 3 & 8 & 3R & FBgn0051226 \\
x151 & CG11911 & 5 & 2 & 7 & 2L & FBgn0031249 \\
x218 & Jon25Biii & 5 & 2 & 7 & 2L & FBgn0031653 \\
x63 & CG9360 & 5 & 2 & 7 & X & FBgn0030332 \\
x246 & Cyp309a1 & 5 & 1 & 6 & 2L & FBgn0031432 \\
x3 & snoRNA:Psi28S-1135f & 5 & 0 & 5 & X & FBgn0083007 \\
x181 & Mal-A7 & 5 & 0 & 5 & 2R & FBgn0033296 \\
x194 & CG3088 & 5 & 0 & 5 & 3L & FBgn0036015 \\
\hline
\end{tabular}
\end{table*}

\begin{table*}[h]
\centering
\caption{Undirected (unoriented) co-regulatory edges discovered by ASCEND ($n=61$).}
\label{tab:undirected_edges}
\small
\begin{tabular}{l l l l}
\hline
Gene A (ID) & Symbol A & Gene B (ID) & Symbol B \\
\hline
x24 & CG42782 & x249 & Cyp4d20 \\
x38 & TotC & x117 & Jon25Bi \\
x60 & Ser6 & x213 & Jon99Cii \\
x102 & CecA1 & x244 & Listericin \\
x164 & CG9259 & x214 & Acp76A \\
x192 & CG6733 & x44 & Lst \\
x207 & CG8012 & x243 & CG33301 \\
x222 & Jhbp12 & x34 & Jon99Fii \\
x100 & CG33290 & x72 & CG4250 \\
x222 & Jhbp12 & x72 & CG4250 \\
x231 & betaTry & x184 & epsilonTry \\
x236 & CG5402 & x182 & RpL10Aa \\
x65 & CHKov1 & x29 & Skadu \\
x106 & mat & x155 & LysE \\
x117 & Jon25Bi & x92 & CG18180 \\
x34 & Jon99Fii & x92 & CG18180 \\
x17 & LysD & x205 & CG7542 \\
x70 & CG30487 & x248 & CG17571 \\
x183 & CG17380 & x250 & Obp56d \\
x182 & RpL10Aa & x40 & Mst84Dc \\
x28 & BomBc1 & x230 & Cht8 \\
x205 & CG7542 & x56 & Jon99Fi \\
x205 & CG7542 & x191 & CG31288 \\
x31 & DptA & x20 & PGRP-SB1 \\
x57 & DptB & x20 & PGRP-SB1 \\
x41 & CG34166 & x69 & Cyp6a9 \\
x2 & CR33319 & x69 & Cyp6a9 \\
x171 & CG15818 & x235 & CG13323 \\
x148 & PGRP-SD & x235 & CG13323 \\
x173 & CG5506 & x136 & CG42397 \\
x84 & CG43120 & x136 & CG42397 \\
x86 & Mtk & x110 & CG8661 \\
x136 & CG42397 & x110 & CG8661 \\
x54 & Srg2 & x108 & CG10910 \\
x110 & CG8661 & x49 & AttC \\
x31 & DptA & x143 & BomS1 \\
x148 & PGRP-SD & x143 & BomS1 \\
x49 & AttC & x143 & BomS1 \\
x220 & tobi & x180 & CG16704 \\
x196 & CG42650 & x140 & CG5767 \\
x54 & Srg2 & x62 & CG13324 \\
x105 & UQCR-6.4L & x9 & sordd2 \\
x174 & LManV & x52 & Jon99Ci \\
x9 & sordd2 & x52 & Jon99Ci \\
x182 & RpL10Aa & x239 & Mal-A8 \\
x142 & GstE7 & x239 & Mal-A8 \\
x221 & Rh2 & x223 & AANATL3 \\
x222 & Jhbp12 & x58 & Mal-B1 \\
x177 & Cyp4p1 & x237 & Jheh1 \\
x58 & Mal-B1 & x237 & Jheh1 \\
x63 & CG9360 & x18 & Cyp6a8 \\
x156 & Cyp4d14 & x23 & Cyp6a2 \\
x18 & Cyp6a8 & x23 & Cyp6a2 \\
x55 & CG43145 & x197 & lncRNA:CR44922 \\
x122 & CG10560 & x76 & CG34215 \\
x225 & Sfp51E & x64 & CG6908 \\
x122 & CG10560 & x64 & CG6908 \\
x41 & CG34166 & x11 & Lsp2 \\
x122 & CG10560 & x11 & Lsp2 \\
x64 & CG6908 & x11 & Lsp2 \\
x186 & \textit{unann.} & x66 & Cyp6w1 \\
\hline
\end{tabular}
\end{table*}

\end{document}